\documentclass[12pt]{article}
\usepackage{fullpage}
\usepackage{graphicx}
\usepackage{subcaption}
\usepackage{enumitem}
\usepackage[boxruled,lined]{algorithm2e}
\usepackage{amsmath,amssymb}
\usepackage{amsthm}
\usepackage{color}
\usepackage{xcolor}
\usepackage[scientific-notation=true]{siunitx}
\usepackage{dsfont}
\usepackage[bottom]{footmisc}
\newcommand{\prob}{\mathbb P}
\newlength{\minipagewidth}
\newcommand{\Real}[1]{ { {\mathbb R}^{#1} } }

\DeclareMathOperator*{\argmin}{argmin}

\newcommand{\field}[1]{\mathbb{#1}}

\newcommand{\R}{\field{R}}

\newcommand{\E}{\field{E}}

\renewcommand{\epsilon}{\varepsilon}
\newcommand{\eps}{\epsilon}

\newcommand{\calZ}{\mathcal{Z}}
\newcommand{\calH}{\mathcal{H}}
\newcommand{\calP}{\mathcal{P}}

\newcommand{\calL}{\mathcal{L}}

\newcommand{\sinc}{{\mathrm{sinc}}}
\newcommand{\One}[1]{\mathds{1}\!\left({#1} \right)}

\newtheorem{lemma}{Lemma}
\newtheorem{theorem}{Theorem}

\newtheorem{remark}{Remark}

\newtheorem{corollary}{Corollary}
\newtheorem{example}{Example}
\newtheorem{assumption}{Assumption}
\newtheorem{definition}{Definition}

\newtheorem{proposition}{Proposition}

\title{\bf Risk Analysis and Design \\ Against Adversarial Actions}
\author{
	Marco C. Campi, Algo Car\`e, Luis G. Crespo, \\ 
	\renewcommand\footnotemark{}
	Simone Garatti, and Federico A. Ramponi\thanks{
		Marco C. Campi, Algo Car\'e, and Federico Ramponi are with the Dipartimento di Ingegneria dell'Informazione -- Universit\`a di Brescia, via Branze 38, 25123 Brescia, Italia. E-mail: {\sf \{marco.campi, algo.care, federico.ramponi\}@unibs.it}; Luis G. Crespo is with the NASA Langley Research Center, MS 308, Hampton, VA, 23681-2199, USA. E-mail: {\sf luis.g.crespo@nasa.gov}; Simone Garatti is with the Dipartimento di Elettronica, Informazione e Bioingegneria -- Politecnico di Milano, piazza Leonardo da Vinci 32, 20133 Milano, Italia. E-mail: {\sf simone.garatti@polimi.it}.}
}

\begin{document}

\date{}
\maketitle

\begin{abstract}
	Learning models capable of providing reliable predictions in the face of adversarial actions has become a central focus of the machine learning community in recent years. This challenge arises from observing that data encountered at deployment time often deviate from the conditions under which the model was trained. In this paper, we address deployment-time adversarial actions and propose a versatile, well-principled framework to evaluate the model's robustness against attacks of diverse types and intensities. While we initially focus on Support Vector Regression (SVR), the proposed approach extends naturally to the broad domain of learning via relaxed optimization techniques. Our results enable an assessment of the model vulnerability without requiring additional test data and operate in a distribution-free setup. These results not only provide a tool to enhance trust in the model's applicability but also aid in selecting among competing alternatives. Later in the paper, we show that our findings also offer useful insights for establishing new results within the out-of-distribution framework. 
\end{abstract}

\noindent
\textbf{Keywords:} Adversarial Learning, Statistical Risk, Learning through Optimization, Support Vector Methods, Statistical Learning Theory.

\section{Introduction}

Recent research demonstrates that machine learning models can be vulnerable to \emph{adversarial examples}. For instance, \cite{Szegedy2014} and \cite{Biggio2013} show that even state-of-the-art neural networks trained on ``clean'' examples can be prone to misinterpret inputs subjected to even slight perturbations. Although misleading adversarial examples can vary with the architecture of the model and the set on which the model has been trained, it has been also recognized that diverse models with different architectures and training sets often misclassify the same adversarial examples, \cite{Szegedy2014}. A comprehensive review of adversarial attacks is provided in \cite{oprea2024adversarial}, which classifies the attacks into two categories: \emph{deployment-time} and \emph{training-time}. Training-time attacks refer to perturbations of the examples used for training, while deployment-time attacks involve perturbations at the time the model is used. For broad overviews, see the recent surveys \cite{Bai2021} and \cite{bajaj2024state}, with the second more specifically targeting image classification.  \\ 

To mitigate deployment-time vulnerabilities, it has been proposed to include artificially generated examples that mimic adversarial actions in the training set, \cite{Goodfellow2015,madry2018towards}, an approach termed ``adversarial training''. Adversarial training has been linked to robust optimization in \cite{SHAHAM2018}, and \cite{Maggioni2024,singla2020survey,anderson2022overview} further discuss robust optimization as a technique for data-driven model robustification. On the other hand, critical evaluations suggest that adversarial training may promote more severe overfitting, leading to increased gaps between training and test accuracy, \cite{Bai2021,Schmidt}. Other works, such as \cite{tsipras2018robustness} and \cite{zhang2019theoretically}, argue that adversarial training can also worsen non-adversarial classification accuracy; similar implications have been investigated for linear regression in \cite{Javanmard2020,ribeiro2024regularization,ribeiro2023overparameterized}. These critiques highlight the need of well-principled methodologies to assess the \emph{risk} (the probability of making mistakes) associated with alternative training strategies, to increase trust and guide selection among them. Although tools like Rademacher complexity, \cite{yin2019rademacher}, and VC dimension, \cite{cullina2018pac}, have been explored for this purpose, theoretical advancements remain limited, leaving an open field for further investigation. \\ 

In this paper, we consider deployment-time attacks and study the ensuing risk using a methodology that is highly structured mathematically. Initially, we focus on Support Vector Regression (SVR), a widely-used regression technique, and then show that our theoretical achievements generalize to the broad framework of \emph{learning via relaxed optimization} techniques. While relaxed optimization is foundational in several Support Vector methods, it also covers vast domains in \emph{decision-making}. Our main contributions are as follows: 
\begin{itemize}
	\item[\bf (i)] we introduce a new, rigorous methodology for evaluating the risk associated with adversarial attacks based on the notion of \emph{complexity}. The user is also allowed to test multiple choices for the adversarial actions, enabling robustness checks against adversarial attacks of varying strengths and types; \\
	\vspace*{-4mm}
	\item[\bf (ii)] the user may robustify the design by perturbing the training examples in a neighborhood of their nominal value. For practical implementation, it is suggested in this paper that the number of perturbed examples be finite for any given example in the training set. Importantly, the proposed risk estimation methodology remains rigorously valid for any envisaged adversarial action, even though robustification only considers a finite perturbation set. This decoupling of algorithmic implementation from risk assessment is a key feature for the applicability of our method. 
\end{itemize}
These results are enabled by a theoretical framework that delves deeply into far-reaching connections between the concepts of risk and complexity, as precisely explained in the paper. \\

As a final contribution:
\begin{itemize}
	\item[\bf (iii)] we show that our adversarial results open new avenues for the study of \emph{out-of-distribution} risk, where training and deployment data are generated according to two distinct distributions. As an example of application, one can think of data generated by a simulator to design a device for use in an uncertain environment. In this context, our results take a significant departure from previous findings, offering a novel and fruitful approach to the problem.
\end{itemize} 

The paper is organized as follows. Support Vector Regression is considered in Section \ref{sec:SVR}, with applications examples (both simulated and with real data) in Section \ref{sec:examples}. Section \ref{sec:opt-relax} deals with learning via relaxed optimization, while the study of the risk for out-of-distribution observations is addressed in Section \ref{sec:ambiguous}. All proofs are postponed until Section \ref{sec:derivations}. 

\section{Adversarial Support Vector Regression}
\label{sec:SVR} 
In this section, we consider predictors built using Support Vector Regression (SVR), and present a theory for the evaluation of the probability with which they make mistakes in the presence of adversarial actions. \\

As detailed in Section \ref{section SVR}, SVR constructs a ``band predictor'': corresponding to each value of an observed input variable $u \in \mathbb{R}^d$, the band predictor returns an evaluation interval for the corresponding output variable $y \in \mathbb{R}$. More specifically, the version of SVR we consider here is the \emph{adjustable-size} SVR introduced in \cite{scholkopf1998shrinking}, and the reader is referred to this reference for a more comprehensive presentation. Paper \cite{Campi21ml} studies the reliability of SVR in a standard setup without adversarial actions. \\

Throughout, ${\cal D} = \{(u_i,y_i)\}_{i=1}^N$ is the {\em training set} used to learn the predictor. Data points $(u_i,y_i)$ are independent draws from a common probability distribution $\prob$ over $\mathbb{R}^d\times\mathbb{R}$ (i.e., they form an i.i.d. -- independent and identically distributed -- sample). In line with  \cite{Campi21ml}, $\mathbb{P}$ is unknown to the user, who has only access to the training set to learn the predictor. As in \cite{Campi21ml}, the only assumption that is made on $\mathbb{P}$ is that, given $u$,  the values of $y$ do not accumulate, as formally defined in the following assumption. 

\begin{assumption}
	\label{non-acc}
	With probability $1$, the regular conditional distribution of $y$ given $u$ admits a density. \hfill$\star$
\end{assumption}

To keep the presentation simple and better focus on the conceptual aspects, we will refer to linear regression in the following. However, we mention that all the results readily extend to the case in which the data are ``lifted'' into a feature space, as is commonly done in machine learning problems using the Reproducing Kernel Hilbert Space (RKHS) technique. Further details on this extension are provided in Remarks \ref{rmk:kernel_trick} and \ref{rmk:kernel_trick_2}. 
\subsection{SVR in the non-adversarial case}
\label{section SVR}
To position our results, we feel advisable to first recall how SVR works in a non-adversarial setup. A SVR predictor is defined by three parameters $w \in \mathbb{R}^d$, $b \in \mathbb{R}$, $\gamma \in \mathbb{R}^+$ (non-negative reals), which we collectively denote as $\theta := (w,b,\gamma)$. A value of $\theta$ defines a \emph{band predictor} $\mathcal{P}(\theta)$ by the following rule 
$$
\mathcal{P}(\theta)=\{(u,y): |y-w^\top u-b|\leq \gamma\}.
$$
Thus, $\mathcal{P}(\theta)$ includes the values of $y$ that deviate from function $w^\top u+b$ no more than $\gamma$. In SVR with adjustable size, the parameter $\theta$ is trained on $\cal D$ by solving the following optimization program ($\tau$ and $\rho$ are two positive hyper-parameters whose value is set by the user):  
\begin{align}
	\label{svr_nominal}
	\min_{w \in \mathbb{R}^d, b \in \mathbb{R}, \gamma \geq 0 \atop \xi_i \geq 0, i=1,\ldots,N} & \quad  \gamma + \tau \| w \|^2 + \rho \sum_{i=1}^N \xi_i \\
	\textrm{\rm subject to:} & \quad |y_i  -  w^\top u_i  - b | - \gamma \leq \xi_i, \ \ i = 1, \ldots ,N. \nonumber
\end{align}
In \eqref{svr_nominal}, the variables $\xi_i$ are used to relax the requirement that all data points lie within the prediction band. Leaving a data point outside corresponds to a penalty in the cost function equal to the vertical distance of the data point from the prediction band (as computed by formula $|y_i  -  w^\top u_i  - b | - \gamma$) multiplied by a user-chosen coefficient $\rho$. It is well known (see, e.g., the discussion presented after Assumption 4 in \cite{Campi21ml}) that \eqref{svr_nominal} certainly admits a solution; when multiple solutions exist, in \cite{Campi21ml} it is suggested to break the tie by selecting the smallest $\gamma^\ast$ and, then, the $b^\ast$ with smallest absolute value ($w^\ast$ is instead always unique). The same tie-break rule is also adopted in this paper when dealing with an adversarial setup.  \\

Denoting by $\theta^\ast$ the solution to \eqref{svr_nominal}, the SVR-trained predictor ${\cal P}(\theta^\ast)$ has been analyzed in \cite{Campi21ml} in relation to the concepts of misprediction and risk given in the following definition. 
\begin{definition}[Misprediction and Risk]
	\label{def:MispredictionAndRisk}
	A predictor $\mathcal{P}(\theta)$ mispredicts $(u,y)$ if 
	$$|y-w^\top u-b|> \gamma$$
	(or, in more compact form, if $(u,y)\notin \mathcal{P}(\theta)$). \\ 
	The {\em risk} of a predictor $\mathcal{P}(\theta)$, denoted by $\textnormal{Risk}(\theta)$, is defined as its probability of misprediction, i.e.,
	$$\textnormal{Risk}(\theta):=\prob \{(u,y)\,:\,|y- w^\top u -b|> \gamma\}$$ 
	(or, in more compact form, $\textnormal{Risk}(\theta):=\prob \{(u,y)\notin \mathcal{P}(\theta)\}$).
	\hfill$\star$
\end{definition}

In \cite{Campi21ml}, a method has been proposed to estimate $\textnormal{Risk}(\theta^\ast)$ using a statistic of the training set ${\cal D}$ called ``complexity''. Evidence is provided that this estimation is accurate, while it does not require any prior knowledge of $\mathbb{P}$. These results show that the data may stand a dual role: (i) training the predictor; while also (ii) providing an accurate evaluation of the ensuing risk. As discussed in \cite{Campi21ml}, such results not only furnish a rigorous ground for an assessment of reliability, they also provide a solid framework for comparing multiple choices of the hyper-parameters and make a selection of their value. 

\begin{remark}[lifting into a feature space] \label{rmk:kernel_trick}  A simple but powerful extension of \eqref{svr_nominal} can be obtained thanks to a lifting into a feature space. To this end, one considers a feature map $\varphi(\cdot)$ that sends the raw measurements $u_i \in \mathbb{R}^d$ into a feature space $\Phi$ with the structure of a Hilbert space. In this context, the training of a SVR is carried out like in \eqref{svr_nominal}, with the only difference that now $w \in \Phi$, and $w^\top u_i$ is replaced by the inner product $\langle w, \varphi(u_i)\rangle$. Interestingly, all operations involved in finding the solution do not ever require to explicitly evaluate $\varphi(u_i)$. Indeed, as shown, e.g., in \cite{Campi21ml}, the optimal solution $w^\ast$ is always obtained as a linear combination of the $\varphi(u_i)$'s, so that in \eqref{svr_nominal} optimization can be confined to considering solutions of the type $w = \sum_k \alpha_k \varphi(u_k)$, $k=1,\ldots,N$. Then, one obtains $\| w \|^2 = \langle \sum_k \alpha_k \varphi(u_k), \sum_ {k'} \alpha_{k'} \varphi(u_{k'}) \rangle = \sum_k \sum_{k'} \alpha_k \alpha_{k'} \langle \varphi(u_k), \varphi(u_{k'}) \rangle$, while the constraints can be rewritten as $|y_i  -  \sum_k \alpha_k \langle \varphi(u_k), \varphi(u_i) \rangle   - b | - \gamma \leq \xi_i$ where all quantities $\langle \varphi(u_k), \varphi(u_{k'}) \rangle$ as well as $\langle \varphi(u_k), \varphi(u_i) \rangle$ can be re-written for short as $K(u_k,u_{k'})$ and $K(u_k,u_i)$. This function $K(\cdot,\cdot)$ is called the ``kernel'', and in actual facts only the kernel needs to be known to carry out the calculations. Additionally, one does not even need to explicitly assign a feature map $\varphi(\cdot)$ and an inner product $\langle \cdot, \cdot \rangle$ from which $K(\cdot,\cdot)$ is obtained by composition. In fact, one can start off by directly assigning a positive definite $K(\cdot,\cdot)$ because theoretical results in Reproducing Kernel Hilbert Spaces ensure that this always implicitly corresponds to allocate a suitable couple $\langle \cdot, \cdot \rangle$ and $\varphi(\cdot)$ for which it holds that $K(\cdot,\cdot) = \langle \varphi(\cdot), \varphi(\cdot) \rangle$. The reader is referred to \cite{Scho98} for more details. 
	\hfill$\star$ 
\end{remark}

\subsection{SVR in an adversarial setup}

We next consider an adversarial setup where a point $(u,y)$ comes with an {\em adversarial region} $A_{(u,y)}\subseteq \mathbb{R}^d\times\mathbb{R}$, and it is desirable that all the points in the region $A_{(u,y)}$ belong to the SVR prediction band. To streamline the presentation, we will focus on the case in which $A_{(u,y)}$ has a fixed shape determined by a set $A\subseteq \mathbb{R}^d\times\mathbb{R}$, shifted according to $(u,y)$:
$$ 
A_{(u,y)}:=\{(u+d_{u},y+d_{y}) \mbox{ with } (d_{u},d_{y})\in A \},
$$ 
or, in more compact form, $A_{(u,y)}=(u,y)+A$. Our results can be easily extended to other choices of $A_{(u,y)}$ that allow for the shape and the size of the adversarial region to depend on the point $(u,y)$; additional discussion is provided in Remark \ref{rmk:generic_adv_set} in Section \ref{sec:opt-relax-theory}. \\ 

\begin{remark}[about the structure of regions $A_{(u,y)}$] In prediction, as well as in classification, adversarial regions often involve perturbing only the input values: the $y$ value, whether continuous or discrete, is estimated from corrupted inputs $u$. This situation is widely found in the literature. For instance, in image recognition one aims to classify cases within specific categories based on images that may have been altered. Critical applications are found in the classification of facial biometric systems, \cite{Sharif2016,Ryu2021}, and in healthcare, where instrumental images are used for diagnosis purposes, \cite{han2020deep,puttagunta2023adversarial}. Our framework, here and in subsequent sections, allows for general adversarial regions that include the case of perturbed inputs, as well as perturbed outputs and other situations of interest. For example, later in Section \ref{sec:opt-relax-apps}, we briefly refer to Support Vector Data Description (SVDD), a technique used to categorize cases of interest. To give a concrete example, suppose that traffic warning signs are photographed in a room from various angles and distances, and SVDD is used to create a class of images associated to the category ``traffic warning sign''. This category can then be loaded into an unmanned car with the purpose of recognizing a warning sign when the car approaches one, and this operation has to be effective even if the warning sign in the street has been perturbed, for example by a sticker attached to it. The flexibility of our framework, as introduced in Section \ref{sec:opt-relax-theory}, also covers this situation.  \hfill $\star$ 
\end{remark}

The following definitions take center stage in our study. 

\begin{definition}[Adversarial misprediction and Adversarial risk]
	\label{def:AdvMispredictionAndAdvRisk}
	A predictor $\mathcal{P}(\theta)$ adversarially mispredicts $(u,y)$ if 
	$$\mbox{there exists a }(\tilde{u},\tilde{y})\in A_{(u,y)}\mbox{ such that }|\tilde{y}-w^\top\tilde{u}-b|> \gamma $$
	(or, in more compact form, if $A_{(u,y)}\nsubseteq \mathcal{P}(\theta)$). \\
	The {\em adversarial risk} of a predictor $\mathcal{P}(\theta)$, denoted $\textnormal{Risk}_A(\theta)$, is the probability of adversarial misprediction, i.e.,
	$$
	\textnormal{Risk}_A(\theta):=\prob \{ (u,y):  \exists (\tilde{u},\tilde{y})\in A_{(u,y)} \mbox{ such that } |\tilde{y}- w^\top\tilde{u}-b|> \gamma  \}
	$$
	(or, in more compact form,   $\textnormal{Risk}_A(\theta):=\prob \{ A_{(u,y)}\nsubseteq \mathcal{P}(\theta)\}$).\hfill$\star$
\end{definition}

Definition \ref{def:AdvMispredictionAndAdvRisk} coalesces to Definition \ref{def:MispredictionAndRisk} when $A=\{0\}$ so that $A_{(u,y)}=(u,y)+\{0\}=\{(u,y)\}$. Thus, the symbol  $\textnormal{Risk}(\theta)$ can be used as a shorthand for $\textnormal{Risk}_{\{0\}}(\theta)$. \\ 

We shall provide results to accurately upper and lower bound the adversarial risk without any additional information behind the use of the training set.  Before making this statement rigorous in the form of a theorem, we generalize the algorithm \eqref{svr_nominal} so as to robustify SVR predictors against adversarial actions. Our results will make reference to this generalization, which contains \eqref{svr_nominal} as a particular case. \\

In principle, a predictor that is more robust against adversarial actions could be obtained by replacing each $i$-th constraint in \eqref{svr_nominal}, i.e., 
$$
|y_i  - w^\top u_i   - b | - \gamma \leq \xi_i,
$$ 
with its adversarial counterpart
$$
|\tilde{y}  - w^\top\tilde{u}   - b | - \gamma \leq \xi_i, \quad \forall(\tilde{u},\tilde{y})\in A_{(u_i,y_i)}.
$$
However, $A$ contains typically infinitely many points, and this formulation would yield a semi-infinite optimization problem, which is known to be much harder to solve than \eqref{svr_nominal}. Thus, we consider the computationally tractable approach of replacing  $A_{(u,y)}$ with a finite subset $\widehat{A}_{(u,y)}=(u,y)+\widehat{A}$, where  $\widehat{A}=\{(d^{(j)}_{u},d^{(j)}_{y})\}_{j=1}^M$ is a finite approximation of $A$ formed by $M$ points taken from $A$.\footnote{For a relaxation of the requirement that $\widehat{A}$ is contained in $A$, a relaxation that is useful in various contexts later explained in the paper, see Section \ref{sec:outersample}.} Figure \ref{fig:adv_sets_approx} depicts examples of possible choices of $\widehat{A}$.
\begin{figure}
	\centering
	\includegraphics[width=0.8\columnwidth]{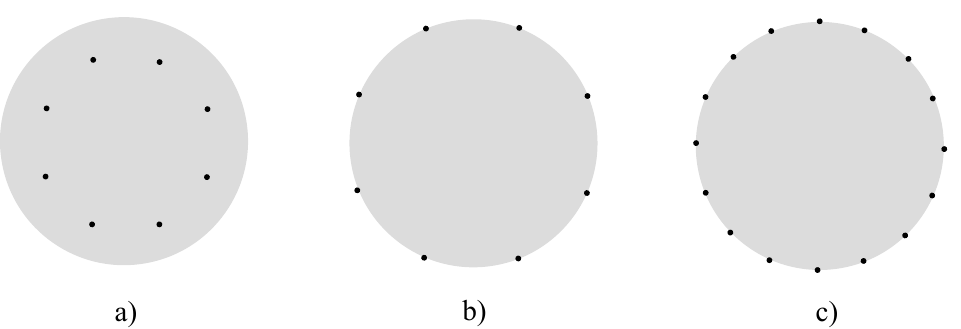}
	\caption{Some choices of $\widehat{A}$ for a ball-shaped adversarial region $A$ (grey area): a) $\widehat{A}$ is in the interior of $A$, which returns less conservative solutions; b) $\widehat{A}$ is tuned to $A$; c) $\widehat{A}$ is a more dense finite set tuned to $A$. 
	}
	\label{fig:adv_sets_approx}
\end{figure}
In introducing this simplification, we are supported by theoretical results that, in spite of the heuristic nature of using $\widehat{A}_{(u,y)}$ in place of $A_{(u,y)}$, provide rigorous evaluations of the risk for the original adversarial region $A_{(u,y)}$. In what follows, we will also use the notation $({\tilde{u}}^{(j)},\tilde{y}^{(j)})$ to indicate the elements of $\widehat{A}_{(u,y)}$, i.e., ${\tilde{u}}^{(j)}:=u+d_u^{(j)}$ and ${\tilde{y}}^{(j)}:=y+d_y^{(j)}$,  $j=1,2,\ldots,M$.  \\

For a given $\cal D$ and a choice of $\widehat{A}$, the adversarially-oriented optimization program is written as follows 
\begin{align}
	\label{svr}
	\min_{w \in \mathbb{R}^d, b \in \mathbb{R}, \gamma \geq 0 \atop \xi_i \geq 0, i=1,\ldots,N} & \quad  \gamma + \tau \| w \|^2 + \rho \sum_{i=1}^N \xi_i \\
	\textrm{\rm subject to:} & \quad |\tilde{y}_i^{(j)} - w^\top \tilde{u}_i^{(j)}  - b | - \gamma \leq \xi_i, \ \,\,j=1,\ldots, M; \;  i = 1, \ldots N. \nonumber
\end{align}
Program \eqref{svr} has a finite number of constraints like the original optimization program \eqref{svr_nominal}, and it can therefore be easily solved with standard numerical solvers. In what follows, we denote by $\theta^\ast_{\widehat{A}} =(w^\ast_{\widehat{A}}, b^\ast_{\widehat{A}}, \gamma^\ast_{\widehat{A}}$) the parameter obtained from \eqref{svr} after possibly breaking ties as indicated above for program \eqref{svr_nominal}. $\mathcal{P}(\theta^\ast_{\widehat{A}})$ is the corresponding predictor. Note also that \eqref{svr_nominal} is recovered from \eqref{svr} when $\widehat{A}=\{ 0\}$; thus, the symbol $\theta^\ast$  can be used as a shorthand for $\theta^\ast_{\{0\}}$.  \\
	
	As compared with \cite{Campi21ml}, the adversarial setup of this section presents two extensions: 
	\begin{itemize}
		\item[\bf (i)] the risk is evaluated with respect to the adversarial region defined through $A$. The shape of this region is dictated by the problem at hand and the user may also want to test out various choices of $A$ to see how robust the design is against adversarial actions of various strengths and types;  
		\item[\bf (ii)] the user may robustify the design by selecting a suitable $\widehat{A}$. Only the choice of $\widehat{A}$ has an impact at an algorithmic level and, normally, $\widehat{A}$ is tuned to a set $A$ that, in the user's mind, captures, and suitably describes, possible adversarial actions. Still, we remark that our results hold true for any choice of $\widehat{A}$ and $A$ (with $\widehat{A} \subseteq A$), so accommodating situations in which, e.g., the user envisages adversarial actions of a certain type and, yet, he is willing to theoretically test the robustness of the design against actions of higher magnitude. One simple example of this situation occurs when the design is done without any adversarial concern (i.e., $\widehat{A} = \{0\}$) and still one wants to test how robust the design is against potential adversarial actions. 
	\end{itemize}
	
	The next section offers a rigorous evaluation, with bounds from above and from below, of the quantity $\textnormal{Risk}_A(\theta^\ast_{\widehat{A}})$, which is the adversarial risk of $\mathcal{P}(\theta^\ast_{\widehat{A}})$. This result represents a notable achievement, also in consideration of the fact that the concept of adversarial risk refers to the whole adversarial regions $A_{(u,y)}$, while training ${\cal P}(\theta^\ast_{\widehat{A}})$ involves considering only the approximated regions $\widehat{A}_{(u_i,y_i)}$. Key to this achievement is the determination of a suitable statistic of the training set, which we call ``adversarial complexity'', from which the adversarial risk can be accurately estimated. 
	
	\begin{remark}[follow-up on Remark \ref{rmk:kernel_trick} about lifting the data into a feature space] \label{rmk:kernel_trick_2}
		
		Similarly to \eqref{svr_nominal}, the adversarially-oriented program \eqref{svr} can be generalized by introducing a lifting $\varphi(\cdot)$, leading to program
		\begin{align} \label{adv-svr-lifted}
			\min_{w \in \calH, b \in \mathbb{R}, \gamma \geq 0 \atop \xi_i \geq 0, i=1,\ldots,N} & \quad  \gamma + \tau \| w \|^2 + \rho \sum_{i=1}^N \xi_i \\
			\textrm{\rm subject to:} & \quad |\tilde{y}_i^{(j)} - \langle w , \varphi(\tilde{u}_i^{(j)}) \rangle  - b | - \gamma \leq \xi_i, \ \,\,j=1,\ldots, M; \;  i = 1, \ldots N, \nonumber
		\end{align}
		which gives the predictor $\mathcal{P}(\theta^\ast_{\widehat{A}}) = \left\{ (u,y): |y-\langle w^\ast_{\widehat{A}} , \varphi(u) \rangle- b^\ast_{\widehat{A}} | \leq \gamma^\ast_{\widehat{A}} \right\}$, where $(w^\ast_{\widehat{A}}, \gamma^\ast_{\widehat{A}}, b^\ast_{\widehat{A}}, \xi_{i,\widehat{A}}^\ast)$ is the solution to \eqref{adv-svr-lifted}. \\
		As in Remark \ref{rmk:kernel_trick}, it is easy to show that the optimal $w^\ast_{\widehat{A}}$ is always obtained as a linear combination of the $\varphi(\tilde{u}_i^{(j)})$'s, so one can search for solutions of the type $w = \sum_{k,h} \alpha_{k,h} \varphi(\tilde{u}_k^{(h)})$, $h=1,\ldots,M$, $k=1,\ldots,N$.
		Introducing the kernel $K(\cdot,\cdot) = \langle \varphi(\cdot), \varphi(\cdot) \rangle$, this leads to the following finite-dimensional rewriting of program \eqref{adv-svr-lifted}
		\begin{align}
			\label{svr2}
			\min_{\alpha_{k,h} \in \mathbb{R},  h = 1,\ldots,M, k=1,\ldots,N \atop  b \in \mathbb{R},\, \gamma \geq 0,\,\xi_i \geq 0, i=1,\ldots,N} & \quad  \gamma + \tau \sum_{k,h} \sum_{k',h'} \alpha_{k,h} \alpha_{k',h'} K(\tilde{u}_k^{(h)},\tilde{u}_{k'}^{(h')}) + \rho \sum_{i=1}^{N} \xi_i \\
			\textrm{\rm subject to:} & \quad |\tilde{y}_i^{(j)} - \sum_{k,h} \alpha_{k,h} K(\tilde{u}_k^{(h)},\tilde{u}_i^{(j)}) - b | - \gamma \leq \xi_i, \ \,\,j=1,\ldots, M; \;  i = 1, \ldots N,\nonumber
		\end{align}
		while the predictor $\mathcal{P}(\theta^\ast_{\widehat{A}})$ can be computed from the solution $(\alpha^\ast_{k,h,\widehat{A}}, \gamma^\ast_{\widehat{A}}, b^\ast_{\widehat{A}}, \xi_{i,\widehat{A}}^\ast)$ of \eqref{svr2} as
		$$
		\mathcal{P}(\theta^\ast_{\widehat{A}})=\left\{ (u,y): |y-\sum_{k,h} \alpha^\ast_{k,h,\widehat{A}} K(\tilde{u}_k^{(h)}, u)- b^\ast_{\widehat{A}} | \leq \gamma^\ast_{\widehat{A}} \right\}.
		$$
		All the results in the following Sections \ref{sec:bounding_adv_risk} and \ref{sec:outersample} are, for the sake of simplicity, presented in the specific setup of \eqref{svr}. However, these results also apply \emph{mutatis mutandis} to predictors $\mathcal{P}(\theta^\ast_{\widehat{A}})$ obtained from \eqref{adv-svr-lifted}. The proofs of the results in Sections \ref{sec:bounding_adv_risk} and \ref{sec:outersample} are given in Section \ref{sec:derivations}, while the applicability of these results to \eqref{adv-svr-lifted} follows from Theorems \ref{th:general1} and \ref{th:general2} in Section \ref{sec:opt-relax}, a section presenting the fairly broad setup of \emph{learning through optimization}, which covers \eqref{adv-svr-lifted} as a specific instance.	\hfill$\star$ 
	\end{remark}
	
	\subsection{Bounding the adversarial risk} \label{sec:bounding_adv_risk}
	
	The adversarial complexity, as defined below, is a quantity that can be computed from the training set $\cal D$, using $A$ and $\widehat{A}$. In statistical language, it is termed a \emph{statistic of the training set $\cal D$}. In light of the main Theorem \ref{main_th_standard} stated below, this quantity plays a key role in the evaluation of the risk. 
	
	\begin{definition}[Adversarial complexity]
		\label{def:adv_complex}
		The adversarial complexity of ${\mathcal P}(\theta^\ast_{\widehat{A}})$, denoted $s_{{A},\widehat{A}}^\ast$, is the number of data points $(u_i,y_i)$ that satisfy at least one of the following two conditions  
		\begin{itemize}
			\item[(i)] $|\tilde{y}_i- {w^\ast_{\widehat{A}}}^\top \tilde{u}_i-b^\ast_{\widehat{A}}|>\gamma^\ast_{\widehat{A}}$ for at least one $(\tilde{u}_i,\tilde{y}_i)\in A_{(u_i,y_i)}$;\footnote{For a discussion on a practical evaluation of (i), see Section \ref{sec:set_con}.}	
			\item[(ii)] $|\tilde{y}_i^{(j)}- {w^\ast_{\widehat{A}}}^\top \tilde{u}_i^{(j)}-b^\ast_{\widehat{A}}|=\gamma^\ast_{\widehat{A}}$ for at least one $(\tilde{u}_i^{(j)},\tilde{y}_i^{(j)}) \in\widehat{A}_{(u_i,y_i)}$.
		\end{itemize}\hfill$\star$
	\end{definition}
	Data points satisfying (i) correspond to mispredictions in the training set. Therefore, through (i) one evaluates the {\em empirical adversarial risk}. On the other hand, the empirical adversarial risk alone does not serve as an effective means to assess $\textnormal{Risk}_{A}(\theta^\ast_{\widehat{A}})$ because the trained predictor can \emph{overfit} the training set. It is therefore reasonable that the empirical adversarial risk needs to be complemented with a quantity measuring the level of adjustment of the predictor to the training set; such a quantity is provided by (ii), which considers the points that ``touch'' the border of the prediction band. \\ 
	
	In preparation of the main theorem, we need to define two functions, $\overline{\eps}(k)$ and $\underline{\eps}(k)$, from $k \in \{0,1,\ldots,N\}$ to $[0,1]$, that will be used to bound the adversarial risk based on $s^\ast_{A,\widehat{A}}$: $\overline{\eps}(s^\ast_{A,\widehat{A}})$ and $\underline{\eps}(s^\ast_{A,\widehat{A}})$ will be, respectively, the upper and lower bound on the adversarial risk. Interestingly, these functions are the same as those used in \cite{Campi21ml} (even though, in \cite{Campi21ml}, they are not computed with the adversarial complexity as their argument). It is a fact that the theory of this paper covers, and aptly generalizes, the non-adversarial theory of \cite{Campi21ml}, which can be recovered by the choices $A = \{0\}$ and $\widehat{A} = \{0\}$. Defining $\overline{\epsilon}(k)$ and  $\underline{\epsilon}(k)$ requires that the user chooses a parameter $\beta\in(0,1)$; as we shall see, the value $1-\beta$ plays the role of a confidence and $\beta$ is often set to a very low value (e.g., $10^{-6}$).
	
	\begin{definition}[Risk-bounding functions $\overline{\eps}(k)$ and $\underline{\eps}(k)$] Given a value in $(0,1)$ of $\beta$ (confidence parameter), for any $k=0,1,\ldots,N-1$ consider the polynomial equation in the $t$ 	variable
		\begin{equation}
			\label{pol_eq-for-eps(k)-relax}
			{N \choose k}t^{N-k} - \frac{\beta}{2N} \sum_{i=k}^{N-1} {i \choose
				k}t^{i-k} - \frac{\beta}{6N} \sum_{i=N+1}^{4N} {i \choose k}t^{i-k}  = 0,
		\end{equation}
		and, for $k = N$, consider the polynomial equation in the $t$ variable
		\begin{equation}
			\label{pol_eq-for-eps(N)-relax}
			1 - \frac{\beta}{6N} \sum_{i=N+1}^{4N} {i \choose N}t^{i-N}  = 0.
		\end{equation}
		In Section 6.1 of \cite{Garatti22} it is shown that, for any $k=0,1,\ldots,N-1$, equation 	\eqref{pol_eq-for-eps(k)-relax} has exactly two solutions in $[0,+\infty)$, which we denote with $\underline{t}(k)$ and	$\overline{t}(k)$ ($\underline{t}(k) \leq \overline{t}(k)$); instead, equation \eqref{pol_eq-for-eps(N)-relax} has only one solution in $[0,+\infty)$, which we denote with $\overline{t}(N)$, while we define	$\underline{t}(N) = 0$.\\ 
		Functions $\overline{\eps}(k)$ and $\underline{\eps}(k)$ are defined as follows: $\overline{\eps}(k) := 1 - \underline{t}(k)$ and $\underline{\eps}(k) := \max \{0,1 - \overline{t}(k)\}$, $k=0,1,\ldots,N$. \hfill$\star$
	\end{definition}
	The zeros of \eqref{pol_eq-for-eps(k)-relax} and \eqref{pol_eq-for-eps(N)-relax} can be efficiently computed using the numerical procedure given in the Appendix B.2 of \cite{CompressionGeneralizationandLearning2023}. For evaluations of $\overline{\eps}(k)$ and $\underline{\eps}(k)$ for specific values of $N$ and $\beta$, the reader is also referred to \cite{Campi21ml}. Moreover, the following explicit formulas, whose derivation can be found in \cite{CompressionGeneralizationandLearning2023}, help gain insight on how functions $\overline{\eps}(k)$ and $\underline{\eps}(k)$ behave for increasing values of the sample size $N$: for any $N$ and all $k \in \{0,1,\ldots,N\}$, it holds that
	\begin{align*}
		\overline{\eps}(k) & \leq \frac{k}{N}+2\frac{\sqrt{k+1}}{N}\left(\sqrt{\ln({k+1})}+4\right)+ 2\frac{\sqrt{k+1}\sqrt{\ln{\frac{1}{\beta}}}}{N}+\frac{\ln{\frac{1}{\beta}}}{N},  \\
		\underline{\eps}(k) & \geq \frac{k}{N}-3\frac{\sqrt{k+1}}{N}\left(\sqrt{\ln({k+1})}+2\right)-3\frac{\sqrt{k+1}\sqrt{\ln{\frac{1}{\beta}}}}{N}.
	\end{align*}
	These formulas show that the upper bound and the lower bound merge on the line $k/N$ as $N$ tends to infinity at a rate that goes to zero as $1/N$ for any fixed $k$ and as $\sqrt{\ln(N)}/\sqrt{N}$ uniformly over $k \in \{0,1,\ldots,N\}$. \\ 
	
	We are now ready to state the main result of this section: for any possible choice of $A$ and of $\widehat{A} \subseteq A$, the {\em adversarial risk} of the predictor  ${\cal P}(\theta^\ast_{\widehat{A}})$ belongs to the interval $[\underline{\eps}(s_{A,\widehat{A}}^\ast), \overline{\eps}(s_{A,\widehat{A}}^\ast)]$ with high confidence $1-\beta$. The confidence indicates an upper bound on the probability of observing a ``poor'' training set, one which leads to an inaccurate assessment of the actual risk. 
	The result holds true for any $\prob$ (\emph{distribution-free} result). 
	\begin{theorem}
		\label{main_th_standard}
		Under Assumption \ref{non-acc}, it holds that 
		\begin{equation} 
			\label{eq:concentration-VS-relax}
			\prob^N \{\,{\cal D} \,:\, \underline{\eps}(s_{A,\widehat{A}}^\ast)\leq \textnormal{Risk}_A(\theta^\ast_{\widehat{A}})\leq \overline{\eps}(s_{A,\widehat{A}}^\ast)\,\,\}\,\geq\,1-\beta,
		\end{equation}
		where ${\cal P}(\theta^\ast_{\widehat{A}})$ is the SVR predictor obtained from \eqref{svr} and $s_{A,\widehat{A}}^\ast$ is its adversarial complexity according to Definition \ref{def:adv_complex}. \hfill$\star$
	\end{theorem}
	\begin{proof}
		See Section \ref{sec:derivations}.
	\end{proof}
	Equation \eqref{eq:concentration-VS-relax} contains $\prob$ twice, once as $\prob^N$ and also implicitly through the definition of $\textnormal{Risk}_A(\theta)$, see Definition \ref{def:AdvMispredictionAndAdvRisk}. However, to apply the theorem this probability need not be known: using \eqref{svr}, one computes $\theta^\ast_{\widehat{A}}$, which gives ${\cal P}(\theta^\ast_{\widehat{A}})$. This depends on the choice of $\widehat{A}$. Then, the complexity $s_{A,\widehat{A}}^\ast$ is evaluated from $\cal D$, $\widehat{A}$ and $A$, and the value of $s_{A,\widehat{A}}^\ast$ is plugged into functions $\overline{\eps}(k)$ and $\underline{\eps}(k)$ to obtain upper and lower bounds for $\textnormal{Risk}_A(\theta^\ast_{\widehat{A}})$. These bounds are guaranteed by the theorem to hold with confidence at least $1-\beta$ regardless of the probability $\prob$ by which the data points are drawn. 
	
	\begin{remark}[deployment-time and training-time risk] The setup introduced in this and later sections, where a training set $\cal D$ from $\prob$ is used to build an SVR model using \eqref{svr}, aligns with the concept of \emph{deployment-time} attack discussed in the introduction. While deployment-time attacks are the primary focus of this paper, it is worth noting that the results derived in this context may have a say also for certain \emph{training-time} attacks. Consider for instance a scenario in which training examples are corrupted according to a deterministic rule that maps any $(u,y)$ to a $(u',y')$ within a distance at most $h$ from $(u,y)$. Through this transformation, the probability $\prob$ is also mapped to a new probability $\prob'$. If we now interpret $(u',y')$ and $\prob'$ as the original $(u,y)$ and $\prob$, then the adversarial risk associated to a ball $A$ of radius $h$ can be used to upper bound the probability of misprediction in this training-time attack setup. \hfill $\star$
	\end{remark}
	
	\subsection{The case $\mathbf{\widehat{A} \not \subseteq A}$}
	\label{sec:outersample}
	So far, it has been assumed that $\widehat{A}$ is contained in $A$. In this section, we relax this assumption and allow $\widehat{A}$ to include elements outside $A$. This also covers the case when one makes an adversarially-oriented design and then wants to assess its risk when no adversarial actions take place (so that $A = \{0\}$). When  $\widehat{A} \not \subseteq A$, our theory is able to provide rigorous upper bounds to the risk, however no lower bounds can be established for reasons that will become clear from the proof of the result. To cover the present situation, we need to introduce a generalized definition of adversarial complexity. 
	
	\begin{definition}[Adversarial complexity -- general definition]
		\label{def:adv_complex_outer}
		The adversarial complexity of ${\mathcal P}(\theta^\ast_{\widehat{A}})$, denoted $s_{{A},\widehat{A}}^\ast$, is the number of data points $(u_i,y_i)$ that satisfy at least one of the following three conditions  
		\begin{itemize}
			\item[(i)] $|\tilde{y}_i- {w^\ast_{\widehat{A}}}^\top \tilde{u}_i-b^\ast_{\widehat{A}}|>\gamma^\ast_{\widehat{A}}$ for at least one $(\tilde{u}_i,\tilde{y}_i)\in A_{(u_i,y_i)}$
			\item[(ii)] $|\tilde{y}_i^{(j)}- {w^\ast_{\widehat{A}}}^\top \tilde{u}_i^{(j)}-b^\ast_{\widehat{A}}|=\gamma^\ast_{\widehat{A}}$ for at least one $(\tilde{u}_i^{(j)},\tilde{y}_i^{(j)}) \in\widehat{A}_{(u_i,y_i)}$ 
			\item[(iii)]  $|\tilde{y}_i^{(j)}- {w^\ast_{\widehat{A}}}^\top \tilde{u}_i^{(j)}-b^\ast_{\widehat{A}}|>\gamma^\ast_{\widehat{A}}$ for at least one $(\tilde{u}_i^{(j)},\tilde{y}_i^{(j)}) \in\widehat{A}_{(u_i,y_i)}$.
		\end{itemize}\hfill$\star$
	\end{definition}
	Notice that this definition coincides with Definition \ref{def:adv_complex} when  $\widehat{A}\subseteq A$ because in this case (iii) implies (i).

	\begin{theorem}
		\label{main_th_outer}
		Without the requirement that $\widehat{A}$ is a subset of $A$, under Assumption \ref{non-acc}, it holds that 
		\begin{equation} 
			\label{eq:concentration-VS-outer}
			\prob^N \{\,{\cal D} \,:\,  \textnormal{Risk}_A(\theta^\ast_{\widehat{A}})\leq \overline{\eps}(s_{A,\widehat{A}}^\ast)\,\,\}\,\geq\,1-\beta,
		\end{equation}
		where ${\cal P}(\theta^\ast_{\widehat{A}})$ is the SVR predictor obtained from \eqref{svr} and $s_{A,\widehat{A}}^\ast$ is its adversarial complexity according to the general Definition \ref{def:adv_complex_outer}. \hfill$\star$
	\end{theorem}
	\begin{proof}
		See Section \ref{sec:derivations}.
	\end{proof}
	It is worth noticing that  
	Theorem \ref{main_th_outer} implies a bound for the (non-adversarial) risk of  ${\cal P}(\theta^\ast_{\widehat{A}})$. In fact, recalling that $\textnormal{Risk}(\theta) = \textnormal{Risk}_{\{0\}}(\theta)$, applying Theorem \ref{main_th_outer} with $A=\{0\}$ yields 
	
	\begin{corollary}[Bound for the non-adversarial risk]
		Under Assumption \ref{non-acc}, we have that 
		\begin{equation} 
			\prob^N \{\,{\cal D} \,:\,  \textnormal{Risk}(\theta^\ast_{\widehat{A}})\leq \overline{\eps}(s_{\{0\},\widehat{A}}^\ast)\,\,\}\,\geq\,1-\beta,
		\end{equation}
		where ${\cal P}(\theta^\ast_{\widehat{A}})$ is the SVR predictor obtained from \eqref{svr} and $s_{\{0\},\widehat{A}}^\ast$ is computed
		using the general Definition \ref{def:adv_complex_outer}. \hfill$\star$ 
	\end{corollary}
	
	\subsection{Set containment condition} \label{sec:set_con}
	
	While conditions (ii) and (iii) in the definitions of adversarial complexity (Definitions \ref{def:adv_complex} and \ref{def:adv_complex_outer}) consist in a simple verification of an inequality for a finite number of cases, condition~(i) entails determining whether the predictor ${\cal P}(\theta^\ast_{\widehat{A}})$ contains a given adversarial region, which might be computationally nontrivial. An optimization-based strategy to make this determination is presented here for the significant case of ball-shaped adversarial regions of the type
	\begin{equation}
		A(c,r) = \{(u,y): \|(u,y)-c\|\leq r\},\label{set}
	\end{equation}
	where: $\|\cdot\|$ is any norm, $c=(c_u, c_y)$, with $c_u\in \Real{d}$ and $c_y\in {\mathbb R}$, is the center, and $r \geq 0$ is the radius. This section has only a significance for the practical implementation of the method, and it can be skipped without any loss of continuity in the conceptual contents of the paper. \\ 
	
	As is clear, one needs first to verify whether $c \notin {\cal P}(\theta^\ast_{\widehat{A}})$: if this is the case one can immediately conclude that $A(c,r)$ is not all contained in ${\cal P}(\theta^\ast_{\widehat{A}})$. If instead $c \in {\cal P}(\theta^\ast_{\widehat{A}})$, then one can proceed by computing
	\begin{equation}
		(\overline{u},\overline{y}) = \argmin_{(u,y) \in \Real{d} \times \Real{}} \{ \|(u,y)-c\| : y - (w^\ast_{\widehat{A}})^\top u - b^\ast_{\widehat{A}} \geq \gamma^\ast_{\widehat{A}} \}, \label{eqa}
	\end{equation}
	and
	\begin{equation}
		(\underline{u},\underline{y}) = \argmin_{(u,y) \in \Real{d} \times \Real{}} \{ \|(u,y)-c\| : y -(w^\ast_{\widehat{A}})^\top u -b^\ast_{\widehat{A}} \leq -\gamma^\ast_{\widehat{A}} \}, \label{eqb}
	\end{equation}
	which are the points on the upper and lower boundaries of ${\cal P}(\theta^\ast_{\widehat{A}})$ closest to $c$ according to the distance induced by $\|\cdot\|$. Letting
	\begin{align*}
		(u^\star,y^\star) =
		\begin{cases}
			(\overline{u},\overline{y}) & \text{if } \| (\overline{u},\overline{y}) - c\|<\| (\underline{u},\underline{y}) - c \| ,\\
			(\underline{u},\underline{y}) & \text{otherwise},
		\end{cases}
	\end{align*}
	we obtain the point on the boundary closest to $c$, which is also called Critical Point (CP), while the set $A(c,r^\star)$ corresponding to $r^\star=\|(u^\star,y^\star)-c\|$ is called the {\em Maximal Set}. Plainly, the adversarial region $A(c,r)$ in \eqref{set} is fully contained in $\mathcal{P}(\theta^\ast_{\widehat{A}})$ if and only if $c\in \mathcal{P}(\theta^\ast_{\widehat{A}})$ and $r \leq r^\star$. The computation of $(\overline{u},\overline{y})$ and $(\underline{u},\underline{y})$ in \eqref{eqa} and \eqref{eqb} amounts to solve convex programs with linear constraints, which is computationally affordable for many standard norms like, e.g., the 2-norm. Moreover, an approach to efficiently compute the CP for either hyper-rectangular or hyper-elliptical regions is available, see \cite{Crespo08}. \\
	
	When, instead, a lifting in a feature space is adopted, 
	the programs corresponding to \eqref{eqa} and \eqref{eqb}, with the obvious adjustments induced by the lifting, may become non-convex, and computing the global minimum may be more difficult. Exceptions are found when the boundaries of the predictor assume particular forms. For instance, this is the case of lifting corresponding to Bernstein polynomials for which sum of squares optimization can be used, see \cite{Lacerda17}. 
	\\
	
	Finally, it is perhaps worth mentioning that, when analytical methods fail, one can always resort to a gridding of $A(c,r)$ and approximately verify whether $A(c,r)$ is contained in ${\cal P}(\theta^\ast_{\widehat{A}})$ by checking whether all the grid points are contained in ${\cal P}(\theta^\ast_{\widehat{A}})$. As is clear, the finer the gridding the better the approximation. Since verifying whether a point is contained in ${\cal P}(\theta^\ast_{\widehat{A}})$ is computationally inexpensive, this approach is often effective. 
	
	\section{Application examples}
	\label{sec:examples}
	
	In this section, the theoretical results so far achieved are illustrated first by means of a synthetic example (Section \ref{sec:toy_ex}) and then by an engineering problem utilizing real data (Section \ref{sec:real_ex}). The two examples serve different purposes: the synthetic example aims at illustrating the utilization of the adversarial generalization theory, while the engineering application example provides experimental validation. The present section is complemented by a discussion of general interest on more comprehensive assessments of the predictor against adversarial actions of varying strength (Section \ref{sec:varying_strength}). 
	
	
	\subsection{Synthetic example} \label{sec:toy_ex}
	
	Consider the data set ${\mathcal D}=\left \{ (u_i,y_i) \right \}_{i=1}^{N}$ of $N=500$ input-output pairs shown in Figure \ref{fig:dataset}. 
	\begin{figure}[t!]
		\centering
		\begin{subfigure}{0.49\textwidth}
			\includegraphics[width=\textwidth]{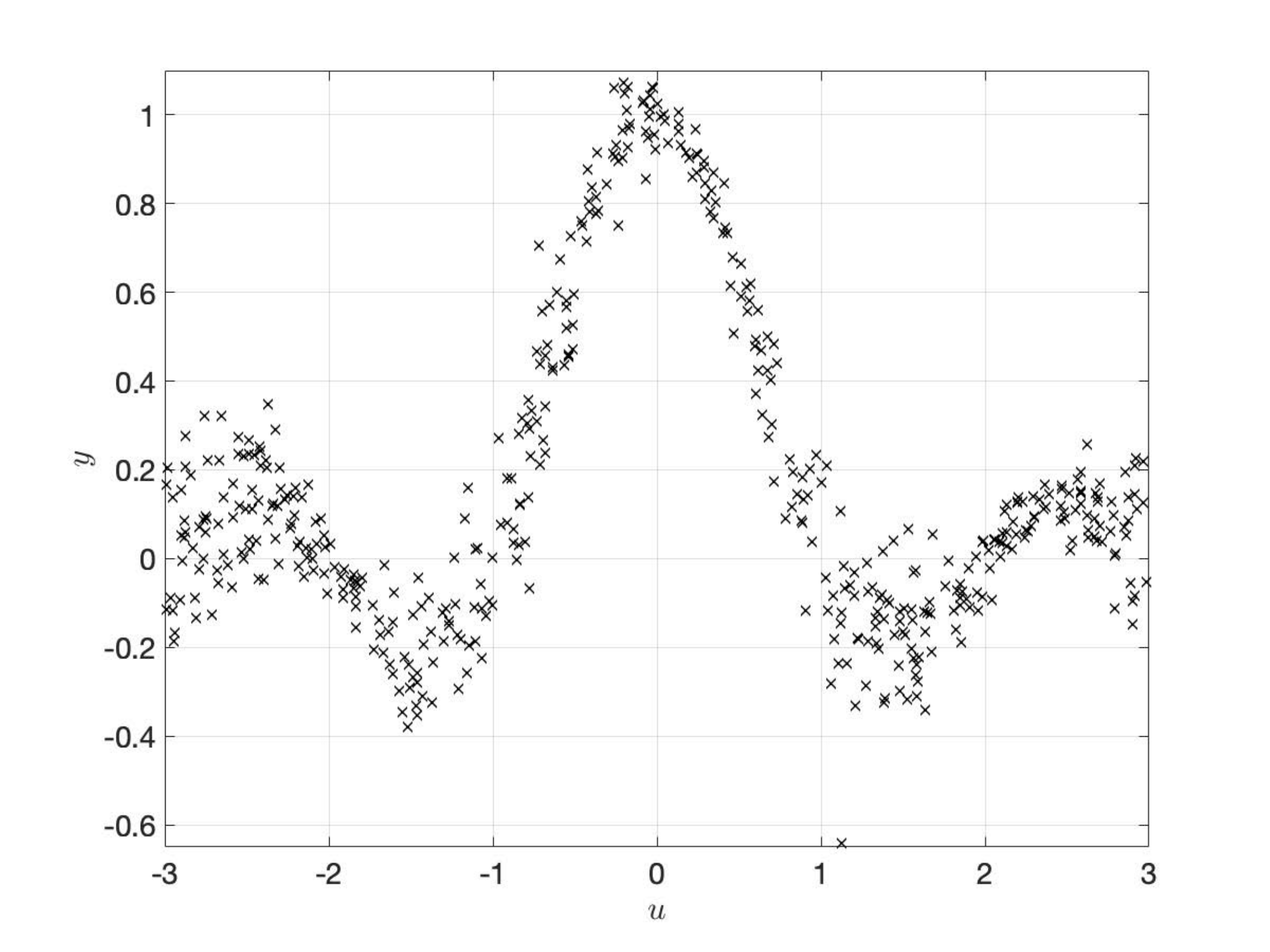}
			\caption{}
			\label{fig:dataset}
		\end{subfigure}
		\hfill
		\begin{subfigure}{0.49\textwidth}
			\includegraphics[width=\textwidth]{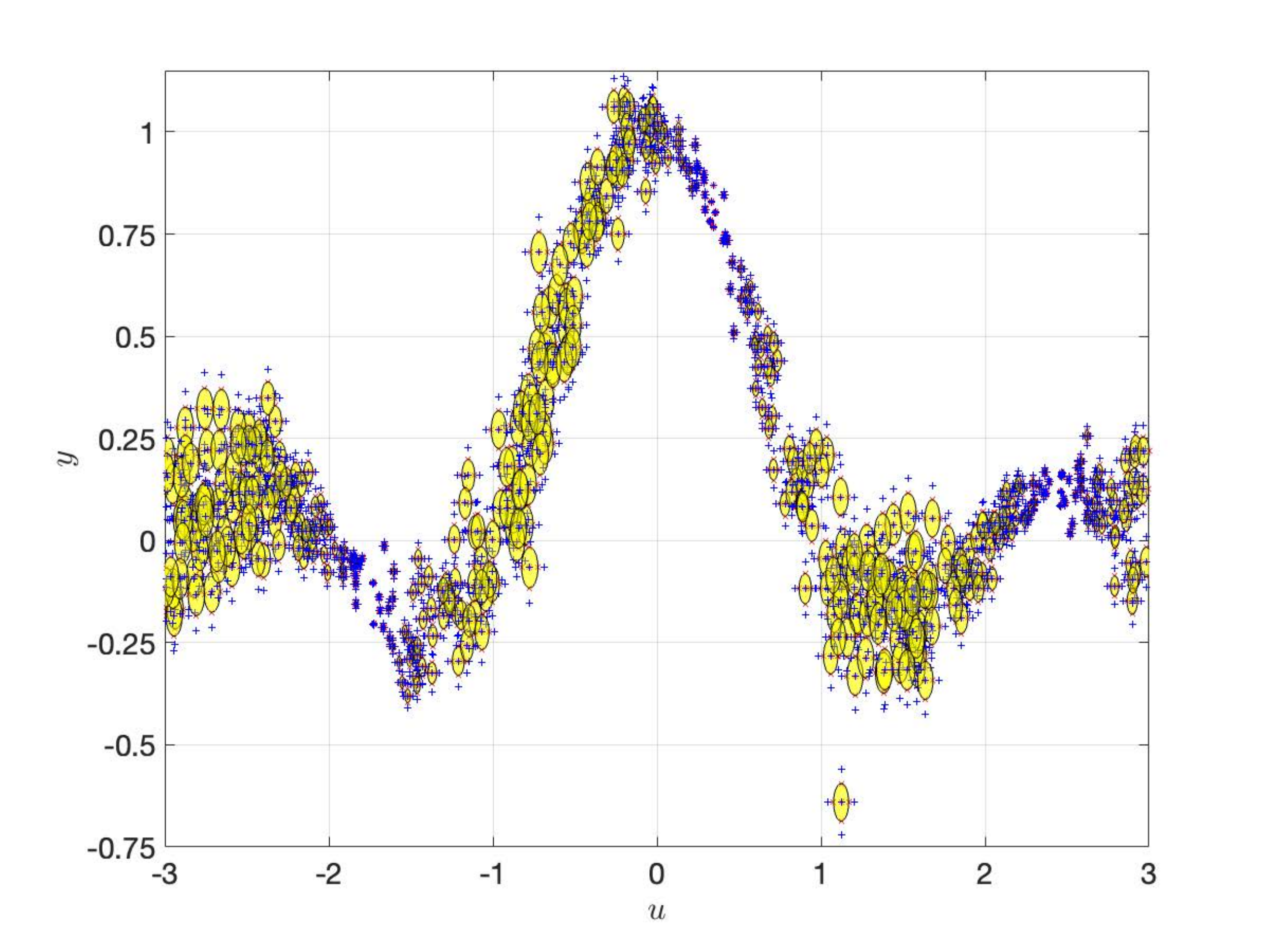}
			\caption{}
			\label{fig:advset}
		\end{subfigure}
		\caption{(a) The data set  ${\mathcal D}=\left \{ (u_i,y_i) \right \}_{i=1}^{N}$. (b) The adversarial regions $A_{(u_i,y_i)}$ (yellow disks) along with the finite sets $\widehat{A}_{(u_i,y_i)}$ when $\widehat{A}_{(u_i,y_i)} \subseteq A_{(u_i,y_i)}$ (red crosses {\color{red} $\times$}) and $\widehat{A}_{(u_i,y_i)} \not\subseteq A_{(u_i,y_i)}$ (blue pluses {\color{blue} $+$}).}
	\end{figure}
	These data were created synthetically by adding input-dependent noise to the $\sinc$ function. In this example, we use program \eqref{svr2} with $\tau=10^{-3}$ and the Gaussian kernel $K(u,u')=\exp(-(u-u')^2/\sigma^2)$ with $\sigma=2$. Two selections of the hyper-parameter $\rho$ are considered, along with various $\widehat{A}_{(u_i,y_i)}$ as specified later. For each computed predictor, the theory of this paper is applied to provide an evaluation of the non-adversarial risk corresponding to $A_{(u,y)} = \{ (u,y) \}$ and of the adversarial risk with regions $A_{(u,y)} = (u,y) + r C$, where $C$ is the unit disk and $r$ depends on $(u,y)$ according to formula $r = |\cos(3u/2 + \pi/3)/20|$ (i.e., the adversarial regions are disks of varying radii-- see Figure \ref{fig:advset} for a representation of $A_{(u_i,y_i)}$, $i=1,\ldots,N$).\footnote{This example is outside the coverage of Section \ref{sec:SVR} because here the adversarial regions do depend on their location. On the other hand, as previously mentioned, the theory of Section \ref{sec:SVR} continues to hold when $A_{(u,y)}$ is not just a translated version of a region $A$, so that the results in Section \ref{sec:SVR} can also be applied in the present context. A precise justification of this fact is provided in Section \ref{sec:opt-relax} as part of a much broader framework in relaxed optimization applicable to many additional problems beyond SVR. 
	} \\
	
	The predictors we compute are: 
	\begin{itemize}
		\item[a.] two non-robust predictors $\mathcal{P}(\theta^\ast_1)$ and $\mathcal{P}(\theta^\ast_2)$ obtained by setting $\widehat{A}_{(u_i,y_i)} = \{ (u_i,y_i)\}$, with $\rho = 4$ and $\rho = 0.05$ respectively;
		\item[b.] two robust predictors $\mathcal{P}(\theta^\ast_3)$ and $\mathcal{P}(\theta^\ast_4)$, corresponding to $\rho = 4$ and $\rho = 0.05$, respectively, and $\widehat{A}_{(u_i,y_i)}$ formed by $M=5$ points, which are:  $(u_i,y_i)$ (the center of $A_{(u_i,y_i)}$) and the top-, bottom-, left-, right-most points on the boundary of $A_{(u_i,y_i)}$. These $\widehat{A}_{(u_i,y_i)}$, $i=1,\ldots,N$, are depicted in Figure \ref{fig:advset} as red crosses {\color{red} $\times$}. Note that $\widehat{A}_{(u_i,y_i)} \subseteq A_{(u_i,y_i)}$ in this case;
		\item[c.] a robust predictor $\mathcal{P}(\theta^\ast_5)$ obtained by setting $\rho = 4$ and $\widehat{A}_{(u_i,y_i)}$ formed by $M=5$ points, which are:  $(u_i,y_i)$ (the center of $A_{(u_i,y_i)}$) and the top-, bottom-, left-, and right-most points on the boundary of $(u_i,y_i) + \frac{3}{2} \cdot r C$; this is an inflated version of the $\widehat{A}_{(u_i,y_i)}$ in point b., providing an outer approximation of $A_{(u_i,y_i)}$. These $\widehat{A}_{(u_i,y_i)}$, $i=1,\ldots,N$, are depicted in Figure \ref{fig:advset} as blue pluses {\color{blue} $+$}. In this case, $\widehat{A}_{(u_i,y_i)} \not\subseteq A_{(u_i,y_i)}$.\footnote{Following up on the previous footnote, we further note that $\widehat{A}_{(u_i,y_i)}$ as defined in points b and c depends on the value of $u_i$ through the parameter $r$, whereas in the treatment of SVR in Section \ref{sec:SVR} such a dependence was not contemplated. However, also this circumstance does not prevent our results from being applied. As a matter of fact, in Section \ref{sec:SVR} the only point where the invariance of $\widehat{A}$ was used was the proof of Proposition~\ref{prop:non-acc}, which established the validity of relation \eqref{relation non accumulation}. Now, Theorems \ref{th:general1} and \ref{th:general2} in Section \ref{sec:opt-relax} can be used to address the present case where $\widehat{A}_{(u_i,y_i)}$ depends on $u_i$; however, in Theorems \ref{th:general1} and \ref{th:general2} relation \eqref{relation non accumulation} is taken as an assumption (Assumption \ref{non-acc-general}) and, hence, one can rightly ask why this assumption is satisfied in the present context. The answer is found in an easy inspection of Proposition~\ref{prop:non-acc}: its thesis, relation \eqref{relation non accumulation}, remains valid even when $\widehat{A}_{(u_i,y_i)}$ exhibits a dependence on $u_i$.} 
	\end{itemize}
	A summary of the results obtained for the five predictors is found in Table \ref{table:nomresults}. The table gives the optimal width $\gamma^\ast$ of the band predictor and the evaluations of the non-adversarial risk $[\,\underline{\epsilon}(s^\ast),\overline{\epsilon}(s^\ast)\,]$ ($s^\ast$ denotes the complexity when $A_{(u,y)} = \{(u,y)\}$, i.e. in the non-adversarial case) and of the adversarial risk $[\,\underline{\epsilon}(s^\ast_{A,\widehat{A}}),\overline{\epsilon}(s^\ast_{A,\widehat{A}})\,]$ obtained by setting $\beta = 10^{-4}$. 
	\begin{table} [t]
		\begin{center}
			\begin{tabular}{|rl|c|c|c|c|c|c|c|}  
				\hline  
				& & $\gamma^\ast$ & $\eta$ & $s^\ast$ & $[\,\underline{\epsilon}(s^\ast),\overline{\epsilon}(s^\ast)\,]$ & $\kappa$ & $s^\ast_{A,\widehat{A}}$ &  $[\,\underline{\epsilon}(s^\ast_{A,\widehat{A}}),\overline{\epsilon}(s^\ast_{A,\widehat{A}})\,]$ \\ [1.5mm]
				\hline 
				{\color{red} $\theta_1^\ast$} &  ${\scriptscriptstyle \left(\rho=4, \ \widehat{A}_{(u_i,y_i)} = \{ (u_i,y_i)\} \right)}$ & {0.3765} & {0} & {7} & ${{[\, 0 ,\, 0.055 \,]}}$ & {24} & {24} & ${{[ \, 0.016,\, 0.11 \, ]}}$  \\[1mm]
				\hline 
				{\color{blue} $\theta_2^\ast$} & ${\scriptscriptstyle \left( \rho=0.05, \ \widehat{A}_{(u_i,y_i)} = \{ (u_i,y_i)\} \right) }$ & {0.1998} & {19} & {21} & ${{[ \, 0.013,\, 0.099 \, ]}}$ & {67} & {67} & ${{[ \, 0.072, \, 0.22 \, ]}}$  \\[1mm]
				\hline 
				$\theta^\ast_3$ & ${\scriptscriptstyle \left( \rho=4, \ \widehat{A}_{(u_i,y_i)} \subseteq A_{(u_i,y_i)} \right)}$ & {0.4256} & {0} & {3} & ${{[\, \text{---}, \, 0.039 \,]}}$ & {7} & {7} & ${{[ \, 0 , \, 0.055 \,]}}$  \\[1mm]
				\hline 
				{\color{cyan} $\theta_4^\ast$} & ${\scriptscriptstyle \left( \rho=0.05, \ \widehat{A}_{(u_i,y_i)} \subseteq A_{(u_i,y_i)} \right)}$ & {0.2472} & {19} & {25} & ${{[\, \text{---},\,0.11 \, ]}}$ & {32} & {32} & ${{[ \, 0.025,\, 0.13 \, ]}}$  \\[1mm]
				\hline 
				{\color{magenta} $\theta_5^\ast$} & ${\scriptscriptstyle \left( \rho=4\, \ \widehat{A}_{(u_i,y_i)} \not\subseteq A_{(u_i,y_i)} \right)}$ & {0.4596} & {0} & {4} & ${{[ \, \text{---} , \, 0.044 \, ]}}$ & {0} & {4} & ${{[ \, \text{---},\,0.044 \, ]}}$  \\[1mm]
				\hline
			\end{tabular}
		\end{center}
		\caption{Performance and risk metrics for the computed predictors. $\gamma^\ast$ is the width of the band predictor, $\eta$ the number of $\widehat{A}_{(u_i,y_i)}$ that are not fully contained in the prediction band, $s^\ast$ the non-adversarial complexity, $[\,\underline{\epsilon}(s^\ast),\overline{\epsilon}(s^\ast)\,]$ the non-adversarial risk bounds, $\kappa$ the number of adversarial regions constructed around the points in the data set that are not fully contained in the predictor, $s^\ast_{A,\widehat{A}}$ the adversarial complexity, and $[\,\underline{\epsilon}(s^\ast_{A,\widehat{A}}),\overline{\epsilon}(s^\ast_{A,\widehat{A}})\,]$ the adversarial risk bounds (in the bounds, the graphic symbol ``$\text{---}$'' indicates that the theory is unable to provide results for the case at hand, this is due to the absence of a lower bound in equation \eqref{eq:concentration-VS-outer}).}
		\label{table:nomresults}
	\end{table}
	The non-adversarial complexity $s^\ast$ and the adversarial complexity $s^\ast_{A,\widehat{A}}$, which are needed to obtain the risk bounds, are also given in Table \ref{table:nomresults}, along with $\eta$, the number of points in the data set outside the prediction band in case of non-robust constructions and the number of regions $\widehat{A}_{(u_i,y_i)}$ that are not fully contained in the prediction band in case of robust constructions, and the number $\kappa$ of adversarial regions $A_{(u_i,y_i)}$ constructed around the points in the data set that are not fully contained in the predictor. 
	The following comments are in order. \\ 
	
	Figure \ref{ex1_ipm12} shows the two non-robust predictors $\mathcal{P}(\theta^\ast_1)$ and $\mathcal{P}(\theta^\ast_2)$.
	\begin{figure}[h!]
		\centering  \includegraphics[clip,width=3.8in]{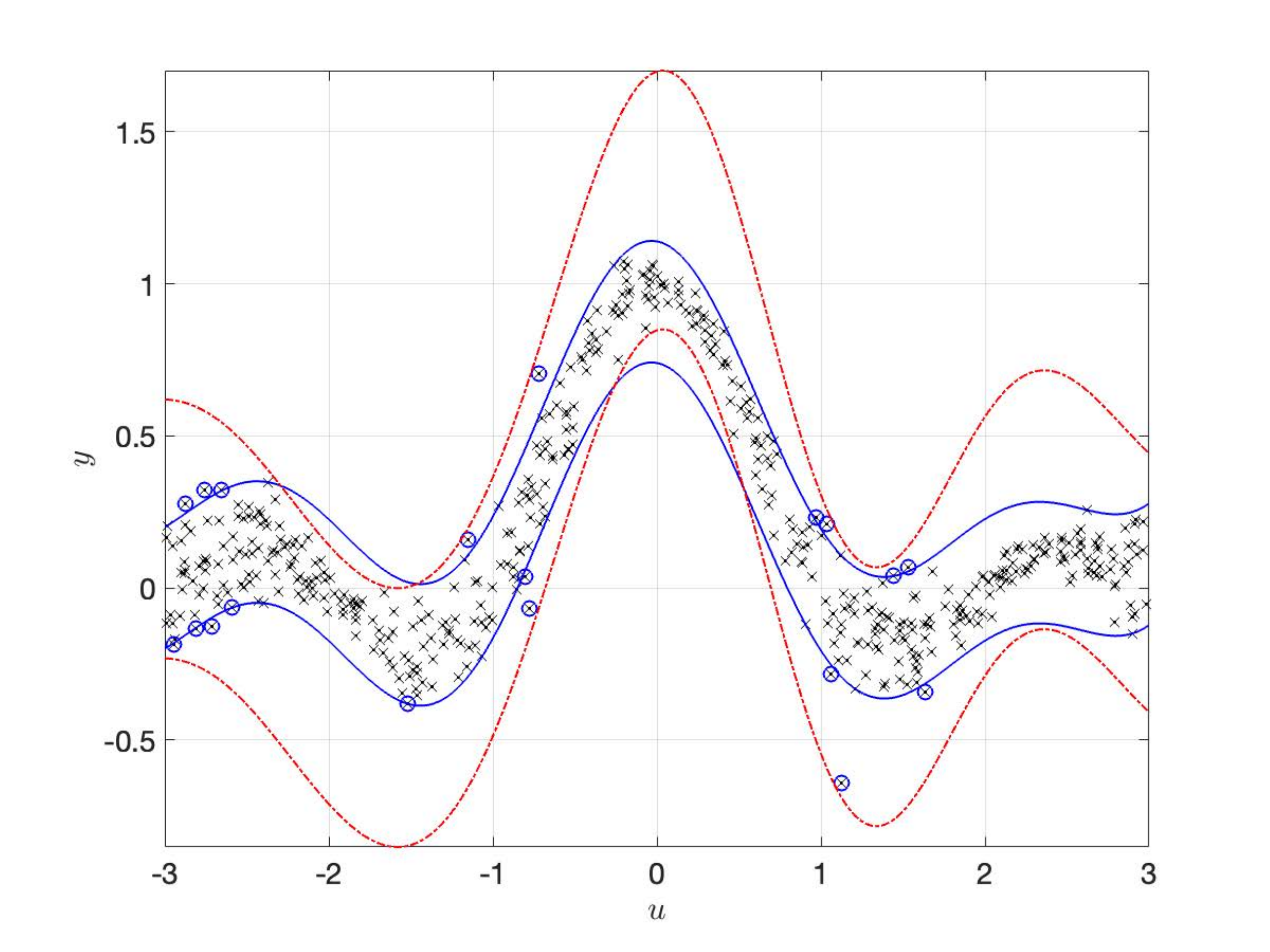}
		\caption{$\mathcal{P}({\color{red} \theta_1^\ast})$ (dashed-dotted-red line) and $\mathcal{P}({\color{blue} \theta_2^\ast})$ (solid-blue line). The data points outside $\mathcal{P}({\color{blue} \theta_2^\ast})$ are marked with a blue circle ({\color{blue} $\circ$}), no data points are outside {$\mathcal{P}({\color{red} \theta_1^\ast})$.}}
		\label{ex1_ipm12}
	\end{figure}
	As expected, when $\rho$ takes a smaller value the predictor width  $\gamma^\ast$ decreases and $\eta$ increases. As a matter of fact, while predictor $\mathcal{P}(\theta^\ast_1)$ encloses the entire data set, $\mathcal{P}(\theta^\ast_2)$ excludes $\eta = 19$ points 
	resulting in $s^\ast \geq 19$ by the very definition of complexity. This generates a higher upper bound on the non-adversarial and adversarial risks (note that functions $\overline{\eps}(k)$ and $\underline{\eps}(k)$ are increasing with $k$), approximately doubling when moving from $\theta_1^\ast$ to $\theta_2^\ast$. \\
	
	As for the adversarial risks of $\mathcal{P}(\theta^\ast_1)$ and $\mathcal{P}(\theta^\ast_2)$, their assessment requires the computation of $\kappa$, the number of adversarial regions that are not fully contained in the predictors. This computation has been performed by means of the procedure described in Section \ref{sec:set_con} and, for the sake of illustration, the resulting maximal sets for $\mathcal{P}(\theta^\ast_1)$ are shown in Figure~\ref{ex1_ms}. \\ 
	\begin{figure}[h!]
		\centering \includegraphics[trim={2cm 1cm 3.5cm 2cm},clip,width=3.5in]{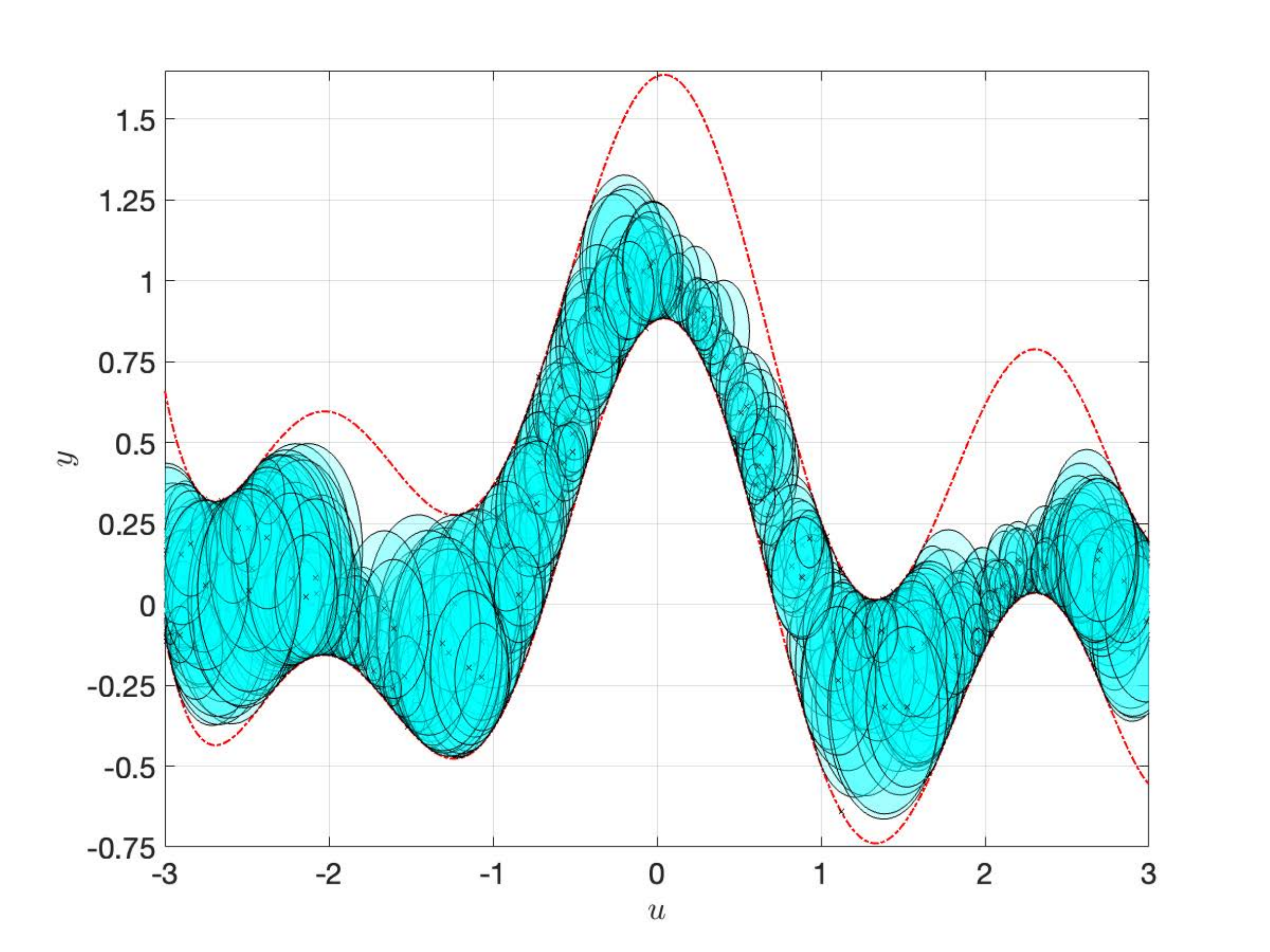}
		\caption{Maximal sets corresponding to $\mathcal{P}(${\color{red} $\theta_1^\ast$}) and all the $(u_i,y_i)$, $i=1,\ldots,N$.}
		\label{ex1_ms}
	\end{figure}
	
	
	The robust predictors $\mathcal{P}(\theta_3^\ast)$ and $\mathcal{P}(\theta_4^\ast)$ are shown in Figure \ref{ex1_ipm345}.
	\begin{figure}
		\centering  \includegraphics[trim={1.5cm 1cm 3cm 2cm},clip,width=3.5in]{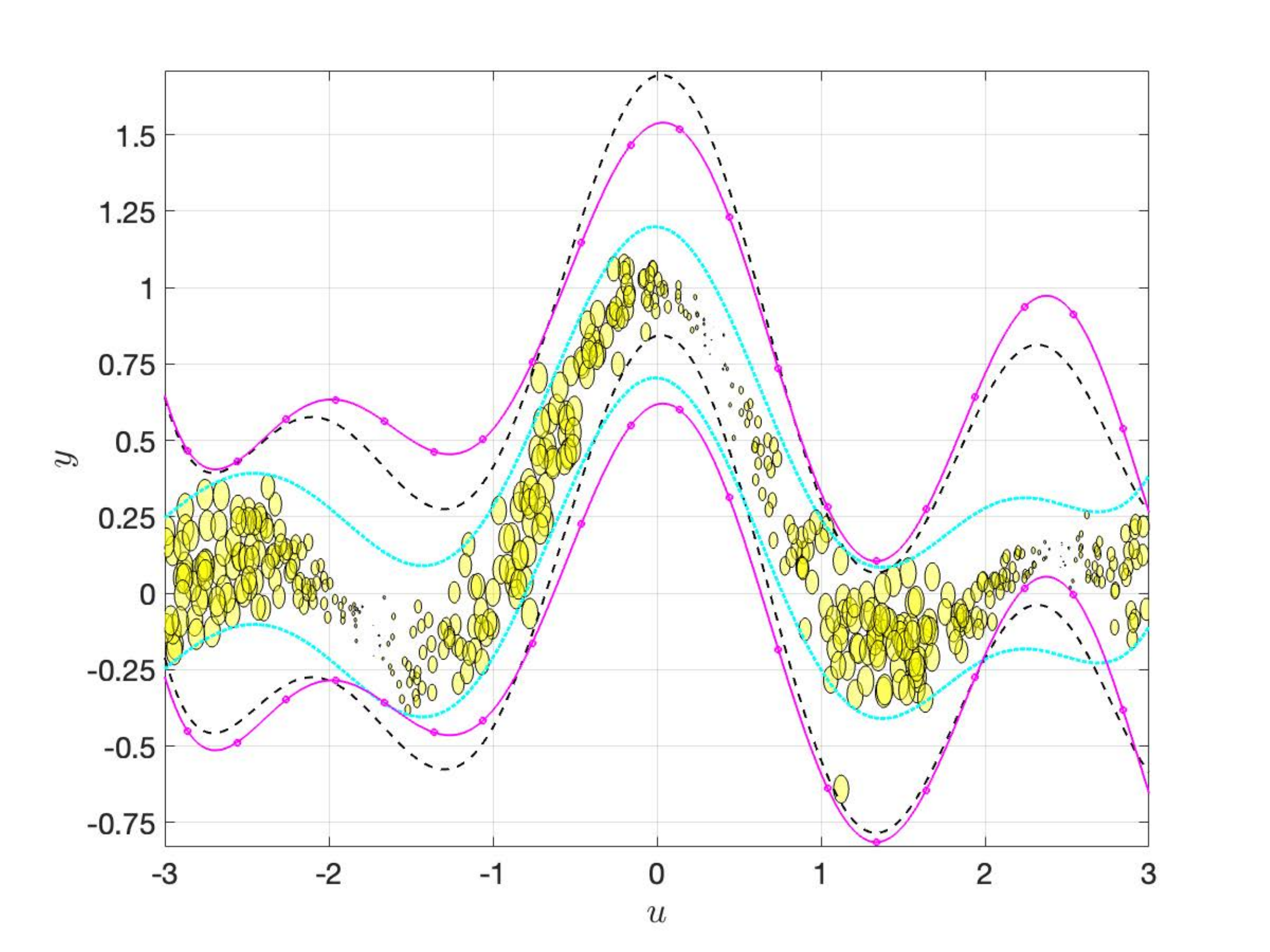}
		\caption{Robust predictors $\mathcal{P}(\theta^\ast_3)$ (dashed-black line), $\mathcal{P}(${\color{cyan} $\theta_4^\ast$}$)$ (dotted-cyan line) and $\mathcal{P}(${\color{magenta} $\theta_5^\ast$}$)$ (solid-circled-magenta line) along with the adversarial regions $A_{(u_i,y_i)}$, $i=1,\ldots,N$.}
		\label{ex1_ipm345}
	\end{figure}
	$\mathcal{P}(\theta_3^\ast)$ encloses all the $\widehat{A}_{(u_i,y_i)}$ (yet, not all the $A_{(u_i,y_i)}$) whereas $\mathcal{P}(\theta_4^\ast)$ fails to enclose $\eta = 19$ of the $\widehat{A}_{(u_i,y_i)}$'s. As before, comparing $\mathcal{P}(\theta_3^\ast)$ and $\mathcal{P}(\theta_4^\ast)$ shows the typical trade-off between performance and risk achieved by the modulation of $\rho$: for instance, focusing on adversarial quantities, selecting $\rho = 0.05$, which lets an additional $5\%$ of the adversarial regions outside the interval ($\kappa$ increases from $7$ to $32$ in a total number of regions equal to $500$), reduces the width by about $58\%$ while increasing the bound on the risk by about $2.36$ times. \\
	
	As is clear, the introduction of an $\widehat{A}_{(u_i,y_i)}$ that includes additional points besides $(u_i,y_i)$ robustifies the design against adversarial actions and, indeed, the robust predictors $\mathcal{P}(\theta_3^\ast)$ and $\mathcal{P}(\theta_4^\ast)$ exhibit adversarial risk bounds lower than those for the corresponding non-robust versions. While Table \ref{table:nomresults} shows that this is not always the case for the non-adversarial risk bounds, it is fair to notice that the use in this case of Theorem \ref{main_th_outer} does not furnish a lower bound, and the upper bound can be somehow conservative. Further, it is important to note that utilizing $\widehat{A}_{(u_i,y_i)}$ adapted to $A_{(u_i,y_i)}$ may help obtain a better trade-off between performance and adversarial risk. This is clear from a comparison between $\mathcal{P}(\theta_1^\ast)$ and $\mathcal{P}(\theta_3^\ast)$, and between $\mathcal{P}(\theta_2^\ast)$ and $\mathcal{P}(\theta_4^\ast)$: a substantial reduction in the adversarial risk bound is obtained while paying a moderate increase of the interval width. \\
	
	Finally, Figure \ref{ex1_ipm345} also depicts the robust predictor $\mathcal{P}(\theta_5^\ast)$, which encloses not only all the $\widehat{A}_{(u_i,y_i)}$ but also all the $A_{(u_i,y_i)}$ (indeed, $\eta = \kappa = 0$). Being $\widehat{A}_{(u_i,y_i)}$ an outer approximation of $A_{(u_i,y_i)}$, $\mathcal{P}(\theta_5^\ast)$ is designed to safeguard the most against adversarial actions, and the entries in Table \ref{table:nomresults} indicate that $\mathcal{P}(\theta_5^\ast)$ has the greatest width and the lowest adversarial risk bound. The comparison between $\mathcal{P}(\theta_3^\ast)$ and $\mathcal{P}(\theta_5^\ast)$, the two most robust predictors, indicates that incrementing the interval width by $8\%$ resulted in an adversarial risk upper bound reduced by about $20\%$. \\
	
	In conclusion, this example shows that \eqref{svr} and \eqref{svr2} are effective and flexible frameworks to obtain competing predictors having different levels of robustness against foreseen adversarial actions. Interestingly, the provided theory allows the user to precisely assess the predictor quality by complementing the predictor width $\gamma^\ast$, which is directly observable, with guaranteed evaluations of the ensuing adversarial risk. The provided characterization ultimately is key to select the predictor that achieves the best overall compromise between contrasting objectives. Importantly, the powerful generalization theory presented in this paper allows the user to achieve this result without resorting to additional data points beyond those used for design. 
	This is paramount in applications where data are valuable and saving data for testing would result in a significant waste of resources.
	
	\subsection{Risk against adversarial actions of various strength} \label{sec:varying_strength}
	
	After determining a suitable predictor for the expected adversarial actions based on the methodology discussed in the previous section, one may also want to further investigate its robustness against adversarial actions of various strength. This involves keeping $\theta^\ast_{\widehat{A}}$ fixed, while computing the complexity $s^\ast_{A,\widehat{A}}$, and the ensuing risk bounds $[\underline{\epsilon}(s^\ast_{A,\widehat{A}}),\overline{\epsilon}(s^\ast_{A,\widehat{A}})]$, for adversarial regions $A^\lambda_{(u,y)} = (u,y) + \lambda \big( A_{(u,y)} - (u,y) \big)$ with $\lambda$ varying over the interval $[0,\lambda_{\max}]$.\footnote{We assume the standard situation in which $A_{(u,y)}$ is star-shaped with respect to $(u,y)$, that is, increasing the inflating parameter $\lambda$ results in an enlarged region that contains regions obtained for smaller values of~$\lambda$.
	} The baseline adversarial case corresponds to $\lambda=1$. Values of $\lambda$ greater than one correspond to expansions (greater adversarial strength) whereas values smaller than one correspond to contractions (smaller adversarial strength) and $\lambda = 0$ is the non-adversarial case. 
	\begin{figure}[t]
		\centering \includegraphics[clip,width=3.5in]{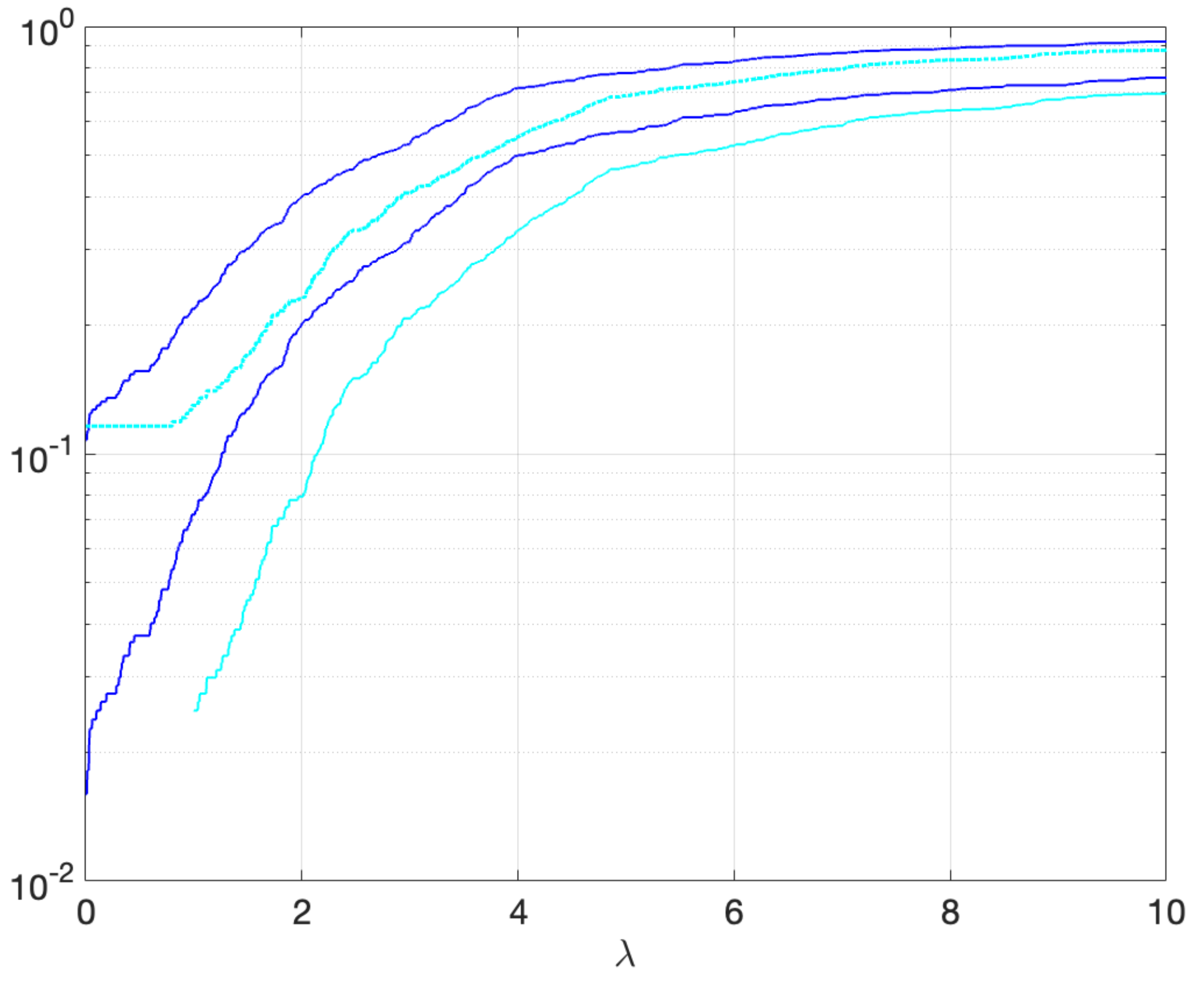}
		\caption{Risk vs. $\lambda$ plot: adversarial risk upper bounds as a function of $\lambda$ for $\mathcal{P}({\color{blue}\theta_2^\ast})$ (solid-blue line) and $\mathcal{P}({\color{cyan}\theta_4^\ast})$ (dotted-cyan line). Note that no lower bound is provided by the theory for the risk of $\mathcal{P}({\color{blue}\theta_2^\ast})$ when $\lambda < 1$.
		}
		\label{risk_curve}
	\end{figure}
	For the sake of illustration, Figure \ref{risk_curve} depicts a plot of $[\underline{\epsilon}(s^\ast_{A,\widehat{A}}),\overline{\epsilon}(s^\ast_{A,\widehat{A}})]$ as a function of $\lambda$ for the predictor $\mathcal{P}({\color{blue}\theta_2^\ast})$ and $\mathcal{P}({\color{cyan}\theta_4^\ast})$. \\
	
	Since $A^{\lambda_1}_{(u,y)} \subset A^{\lambda_2}_{(u,y)}$ for $\lambda_1 < \lambda_2$ (i.e., increasing $\lambda$ yields a family of nested sets), by the very definition of adversarial complexity,\footnote{Note that only condition (i) is affected by $\lambda$, while (ii) and (iii) do not.} we have that $s^\ast_{A,\widehat{A}}$ increases with $\lambda$, and, correspondingly, the risk upper bound increases with $\lambda$ as well. This adheres to obvious qualitative expectations, while a risk vs. $\lambda$ plot like the one in Figure \ref{risk_curve} provides a quantitative determination of this dependency. As examples of use, the plot enables the user to determine the tolerable strength of an adversarial action while maintaining the risk below a given threshold, or to check whether there are critical strength levels at which the risk manifests sudden jumps. As is obvious, the user's preference for a given predictor can also be based on this whole wealth of information.

	\subsection{Engineering application} \label{sec:real_ex}
	
	In-flight loss of control (LOC) is the largest fatal accident category for commercial jet airplane accidents worldwide, see e.g. $  $\cite{Belcastro11}. Aircraft LOC can be described as motion that occurs outside the normal operating flight envelope, not predictably altered by pilot commands, driven by nonlinear effects and coupling, and characterized by disproportionately large responses to small changes in the vehicle's state or oscillatory/divergent behavior, \cite{Belcastro11,Crespo12}. The uncommanded angular rates characterizing these responses seriously compromise the ability to maintain heading, altitude, and wings-level flight. \\
	
	NASA's Langley research center conducted flight experiments to study this phenomenon using the Generic Transport Model (GTM), a 5.5\% dynamically scaled, remotely piloted, twin-turbine aircraft. Some of the experimental data are shown in Figure \ref{ex3_data}, where the output variable is the \emph{lift coefficient} whereas the two input variables are the \emph{angle of attack} and the \emph{sideslip angle}. These data correspond to $16$ flights in a critically upset condition in which the angle of attack is increased progressively until the aircraft stalls followed by a recovering maneuver. The variability in the responses is significant, despite the flights being nominally identical. The goal is to construct an SVR predictor using these data, and quantify its risk. NASA's interest in this experiment is that the resulting predictor, along with uncertainty evaluations, can be used to assess and improve the effectiveness of flight controllers and autopilots during flight upsets, as well as to make flight simulations more realistic. \\
	\begin{figure}[ht!]
		\centering \includegraphics[trim={2cm 0.5cm 1.5cm 0.5cm},clip,width=6in]{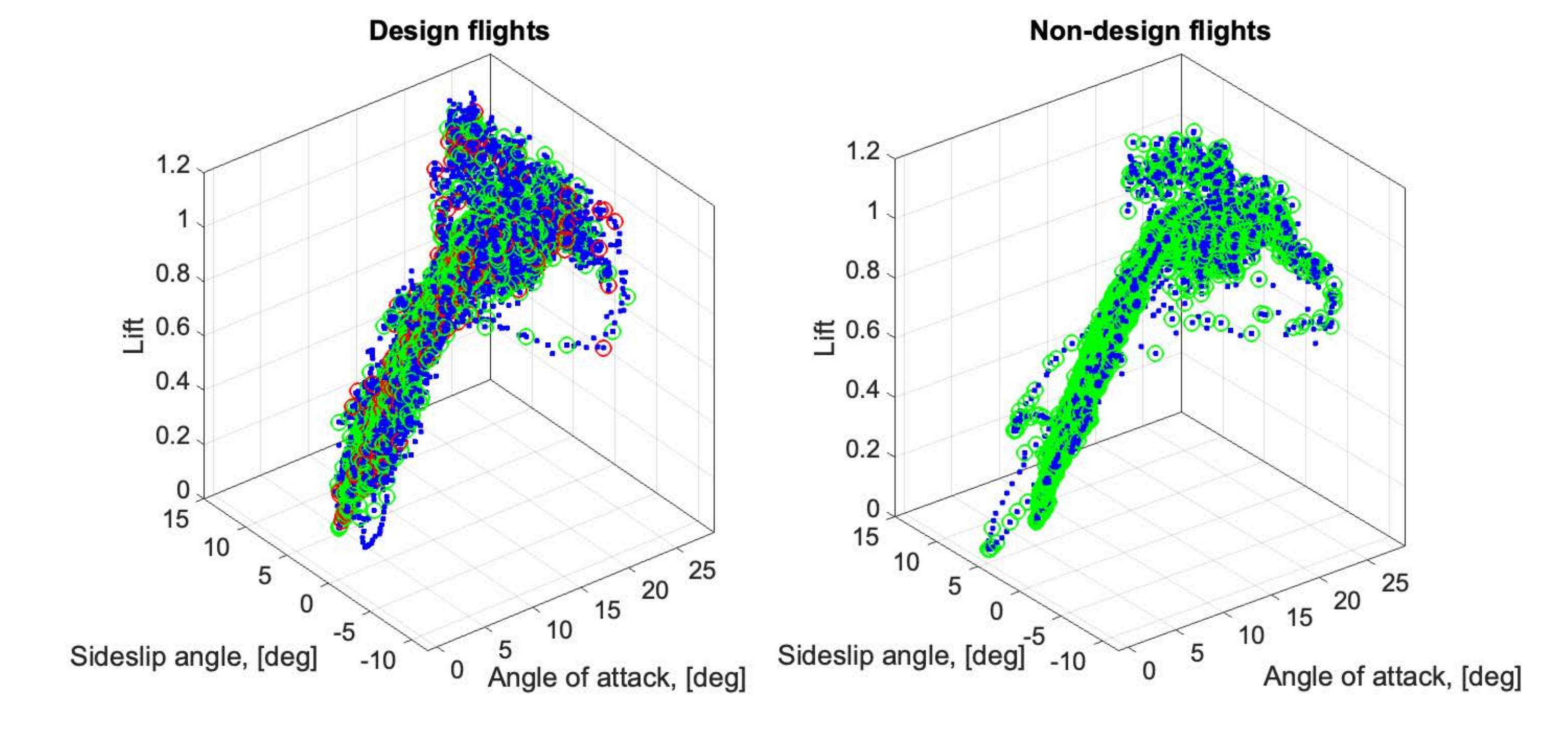}
		\caption{The left subplot shows the dataset for all $12$ design flights (blue dots), the training dataset ({\color{red} $\circ$}) and the test dataset $T_1$  ({\color{green} $\circ$}). The right subplot shows the dataset for all $4$ non-design flights (blue dots) and the test dataset $T_2$  ({\color{green} $\circ$}).}
		\label{ex3_data}
	\end{figure}
	
	Only the data corresponding to the first $12$ flights are used for design purposes, resulting in $36006$ data points. The training set is obtained by randomly selecting $N=1350$ input-output data points from these $36006$ data points, providing a sample that can be considered approximately independent. To empirically validate the risk assessments resulting from the theory, two test datasets were constructed, $T_1$ and $T_2$. Specifically, $T_1$ comprises $5000$ randomly selected data points from the remaining data points in the first $12$ flights, while $T_2$ consists of $5000$ randomly selected data points from the $12002$ data points in the remaining $4$ non-design flights. Thus,  $T_1$ incorporates out-of-sample data points from the same mechanism generating the training set. On the other hand, $T_2$ comes from different flights and, given the considerable variability from flight to flight, $T_2$ can be thought of as containing data points generated from the same mechanism as the training set, but corrupted by some (non-malicious) adversarial action. As for the description of the adversarial action, specific studies revealed that the deviation of data points among flights is typically contained within an ellipsoid $A_\psi$ with axes aligned with the inputs and output, having semi-axes of length $\ell=[0.2,0.5+0.5|\psi|,0.02]$, where $\psi$ is the sideslip angle. Therefore, we considered the adversarial regions $A_{(u,y)} = (u, y) + A_\psi$. \\ 
	
	Two SVR predictors, denoted as $\mathcal{P}(\theta^*_{nr})$ and $\mathcal{P}(\theta^*_{r})$, were computed using (\ref{svr2}) with a feature map $\varphi(\cdot)$ containing fourth-order polynomials. A large value of $\rho$ was used in both cases, which led to data-enclosing predictors (i.e., no data points were left outside the predictors). The non-robust predictor $\mathcal{P}(\theta^*_{nr})$ was obtained by setting $\hat{A}_{(u_i,y_i)}={(u_i,y_i)}$, while the robust predictor $\mathcal{P}(\theta^*_{r})$ was obtained by the choice $\hat{A}_{(u_i,y_i)}={(u_i,y_i)}+\hat{A}_{\psi_i}$, where $\hat{A}_{\psi_i}$ is formed by $M = 15$ points on the surface of $A_{\psi_i}$. \\ 
	
	\begin{figure}[t!]
		\begin{subfigure}{0.2\textwidth}
			\centering
			\includegraphics[trim={2.5cm 1cm 3cm 0cm},width=2.7in]{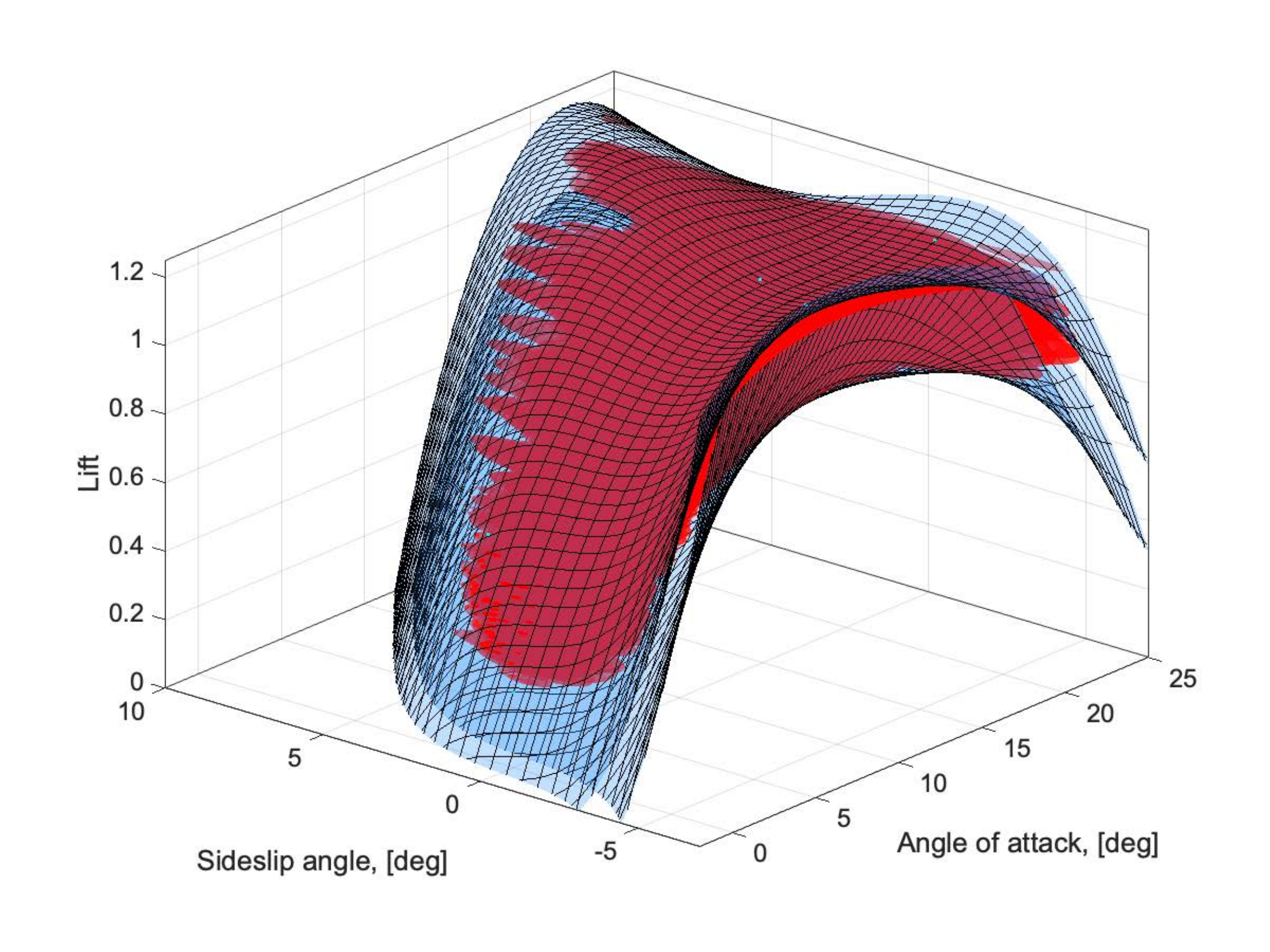}
		\end{subfigure}\hfill
		\begin{subfigure}{0.2\textwidth}
			\centering
			\includegraphics[trim={2.5cm 1cm 3cm 0cm},width=2.7in]{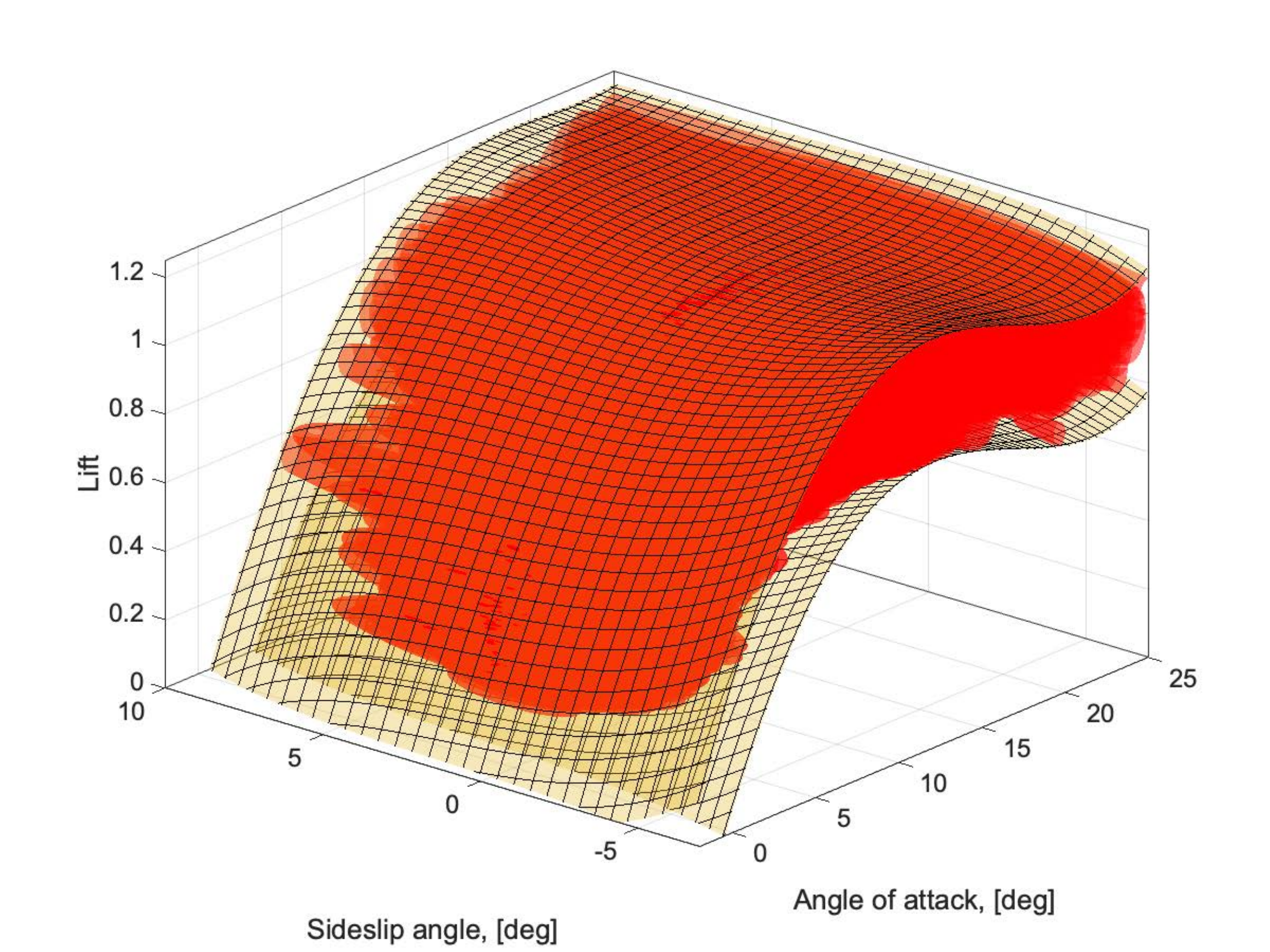}
		\end{subfigure}\hfill
		\label{svrs_aero}
		\caption{ $\mathcal{P}(\theta^*_{nr})$ (left) and $\mathcal{P}(\theta^*_{r})$ (right) with the maximal sets.}
		\label{svrs_aero}
	\end{figure}
	
	Figure \ref{svrs_aero} displays $\mathcal{P}(\theta^*_{nr})$ and $\mathcal{P}(\theta^*_{r})$. For each predictor, the non-adversarial and adversarial complexities $s^*$ and $s^*_{A,\hat{A}}$ were evaluated, and Theorems \ref{main_th_standard} and \ref{main_th_outer} were used to provide risk bounds (for $\beta = 10^{-4}$).\footnote{The polynomial structure of the predictor's boundaries allowed us to find global minima of (\ref{eqa}) and (\ref{eqb}) and, thus, the true maximal sets that were used to compute  $s^*$ and $s^*_{A,\hat{A}}$.} In addition, the out-of-sample empirical frequency of misprediction for the two test datasets $T_1$ and $T_2$ were computed and indicated as $R_{T_1}$ and $R_{T_2}$. 
	Table \ref{ex3table} presents the results. 
	\begin{table} [h!]
		\scriptsize
		\begin{center}
			\begin{tabular}{|c|c|c|c|c|c|c|c|c|c|}  
				\hline 
				& {\footnotesize $\gamma^*$}  & {\footnotesize $s^*$}  &{\footnotesize $[\,\underline{\epsilon}(s^*),\overline{\epsilon}(s^*)\,]$} & $R_{T_1}$ & {\footnotesize \bf{$s^*_{A,\hat{A}}$}} & {\footnotesize $[\,\underline{\epsilon}(s^*_{A,\hat{A}}),\overline{\epsilon}(s^*_{A,\hat{A})}\,]$}  & $R_{T_2}$ & $AR_{T_2}$ \\ [2mm]
				\hline 
				{\color{blue} $\theta_{nr}^*$} & 0.127   & 16 & $[0.27,\,3.17]\times 10^{-2}$ &  $1.36\times 10^{-2}$ & 108 & $[4.8,\,12.0]\times 10^{-2}$   &  $1.92\times 10^{-2}$ &7.6$\times 10^{-2}$ \\[1mm]
				\hline 
				{\color{orange} $\theta_{r}^*$} & 0.150   & 14 & $[-,\,2.93]\times 10^{-2}$ & $0.20\times 10^{-2}$ & 14 &  $[0.20,\,2.93]\times 10^{-2}$   &   $0.34\times 10^{-2}$ & 1.36$\times 10^{-2}$\\[1mm]
				\hline
			\end{tabular}
		\end{center}
		\caption{Performance and risk metrics for $\mathcal{P}(\theta^*_{nr})$ and $\mathcal{P}(\theta^*_{r})$. $\gamma^*$ is the width of the band predictor, $s^*$ the non-adversarial complexity, $[\,\underline{\epsilon}(s^*),\overline{\epsilon}(s^*)\,]$ the non-adversarial risk bounds,  $s^*_{A,\hat{A}}$ the adversarial complexity, $[\,\underline{\epsilon}(s^*_{A,\hat{A}}),\overline{\epsilon}(s^*_{A,\hat{A}})\,]$ the adversarial risk bounds, $R_{T_1}$ and $R_{T_2}$ the empirical frequencies of misprediction on $T_1$ and $T_2$, and $AR_{T_2}$ is obtained as $4R_{T_2}$. 
		}
		\label{ex3table}
	\end{table}
	The table also gives quantity $AR_{T_2}$, a quantity that has an interpretation as explained in the following. A natural way to validate the adversarial bounds in a synthetic example entails drawing additional data points from the underlying data-generating mechanism, computing the adversarial sets corresponding to all these data points, and determining the fraction of them having at least one point outside the predicted interval. 
	In contrast, in an example with real data, one only has empirical data points, in our case belonging to the $4$ non-design flights. In the attempt to realign the empirical results with the theory, one may consider $4$ points from the $4$ non-design flights as draws from an adversarial set, and declare adversarial misprediction if one of them lies outside the predictor. However, this approach comes with a challenge: clustering points in groups of four, so that they can be interpreted as coming from the same adversarial set, is practically unviable. Therefore, we more simply used $AR_{T_2} = 4 R_{T2}$ as an empirical estimate of the adversarial risk. In a sense, this is an overestimate of the risk since the points belonging to the same cluster will independently contribute to the tally, e.g., if $2$ points out of the $4$ are outside the predictor, these two should count as a single misprediction but we are counting them as $2$ mispredictions. On the other hand, we are also underestimating the risk because we are only using $4$ draws out of the infinitely many within the adversarial set. While this is only the best we managed to do, the hope was that the overestimation and underestimation somehow compensated each other, thereby leading to a meaningful estimate. \\
	
	To analyze the results, let us consider first the non-robust predictor $\mathcal{P}(\theta^*_{nr})$. The non-adversarial complexity is quite small relative to the cardinality of the training set, resulting in an accurate evaluation of a moderate non-adversarial risk. The empirical frequency of misprediction $R_{T_1}$ falls within the predicted bounds. 
	Turning to adversarial actions, we see that the adversarial complexity $s^*_{A,\hat{A}}$ is much greater than $s^*$, a sign that many of the adversarial sets $A_{(u_i,y_i)}$ fall outside $\mathcal{P}(\theta^*_{nr})$. 
	The empirical estimate of the risk $AR_{T_2}$ is within the bounds. 
	The robust predictor $\mathcal{P}(\theta^*_{r})$ has a $18\%$ larger width but a much smaller adversarial complexity $s^*_{A,\hat{A}}$. 
	Correspondingly, the adversarial risk upper bound drops by a factor of $4$ compared to the non-robust design.    
	Also in this case $AR_{T_2}$ is within the predicted bounds.

	\section{Learning through optimization
	} 
	\label{sec:opt-relax} 
	
	The optimization program \eqref{svr} serves as a \emph{learning scheme} for constructing band predictors. An important feature is its flexibility, which comes from allowing some data points to lie outside the predictor band via the introduction of the relaxation variables $\xi_i$. In this section, we move to consider general learning schemes based on relaxed optimization of which \eqref{svr} is just a particular instance. Our goal is to demonstrate that the theoretical results we have presented before carry over to this general setup, with significant implications across various fields, including modeling, prediction and classification, as well as broader decision-making contexts such as control design and data-driven actuarial and financial applications. The interested reader is referred to \cite{campi2021Annual} for a comprehensive presentation of the use of relaxed optimization techniques in multiple applied domains, and to \cite{Campi21ml} for results that apply in a non-adversarial context. The presentation of this section is organized as follows. In Section \ref{sec:opt-relax-theory}, we introduce the precise mathematical setup and state the ensuing theoretical results; in turn, Section \ref{sec:opt-relax-apps} provides a brief discussion of specific contexts to which the theory of Section \ref{sec:opt-relax-theory} can be applied. 
	
	\subsection{Adversarial risk generalization results
	} \label{sec:opt-relax-theory}
	
	In this section, data points are indicated with the symbol $\delta$, and are elements of a generic space $\Delta$. For instance, in SVR, $\delta = (u,y)$ and $\Delta = \R^d \times \R$. More generally, a $\delta$ can be an element of a Euclidean space, representing for example the rate-of-return of an investment or, even, it can be an infinite dimensional object, as it happens in classification problems using a waveform as input, for example an ECG (electrocardiogram) tracing to classify a patient. Regardless, at a mathematical level $\Delta$ is just a generic set endowed with a $\sigma$-algebra $\mathcal{G}$ and a probability measure $\prob$, so that $(\Delta,\mathcal{G},\prob)$ is the probability space that models the data generating mechanism. Importantly, nowhere in our treatment it is required that this probability space is known to the user, who only has access to a set of data points drawn from it: ${\cal D} = \{\delta_i\}_{i=1}^N$, where $\delta_i \in \Delta$, $i=1,\ldots,N$, is an i.i.d. sample from $(\Delta,\mathcal{G},\prob)$. \\
	
	For any $\delta \in \Delta$, the corresponding adversarial region 
	is denoted by $A_\delta$, where $A_\delta \subseteq \Delta$. A generic element of $A_\delta$ will be denoted by $\tilde{\delta}$. As in the previous section, we aim at enforcing some level of robustness against adversarial actions by utilizing approximations of finite cardinality of the adversarial regions. Thus, for any $\delta$ we also introduce $\widehat{A}_\delta$, which is a finite set formed by $M$ points of $\Delta$ (i.e., $\widehat{A}_\delta = \{ \tilde{\delta}^{(j)}, \; j=1,\ldots,M \}$ where $\tilde{\delta}^{(j)} \in \Delta$ for all $j$). No constraint on $\widehat{A}_\delta$ relative to $A_\delta$ is enforced and, similarly to Section \ref{sec:SVR}, two results will be obtained, depending on whether $\widehat{A}_\delta \subseteq A_\delta$ or not.
	
	\begin{remark} \label{rmk:generic_adv_set}
		In Section \ref{sec:SVR} we made explicit reference to the case in which $A_\delta$ was obtained as a translated version of a set $A$ (and also $\widehat{A}_\delta$ as a translated version of $\widehat{A}$). This choice was made for simplicity. In the present section we abandon this limitation and allow $A_\delta$ (and $\widehat{A}_\delta$) to change shape and size with $\delta$, which accommodates situations in which an adversary acts selectively depending on the value of $\delta$. The more general results of this section can also be applied to SVR, which is a particular case of the general theory presented herein. 
		\hfill $\star$
	\end{remark}
	Based on the dataset ${\cal D}$, one is asked to select a hypothesis from a set $\Theta$, which is assumed to be a convex set belonging to a linear vector space. $\Theta$ takes manifold interpretations depending on the problem at hand: a $\theta \in \Theta$ may represent the parameter vector of a predictor (as it happens for SVR), or the parametrization of a classifier, or that of a decision in a control problem, \emph{et cetera}. \\
	
	We are interested in hypotheses $\theta^\ast_{\widehat{A}}$ constructed from $\cal D$ by solving the following optimization program (compare with \eqref{svr}) 
	\begin{align} 
		\label{opt-relax}
		\min_{\theta \in \Theta \atop  \xi_i \geq 0, i=1,\ldots,N} & \quad  c(\theta) + \rho \sum_{i=1}^{N} \xi_i \\
		\textrm{\rm subject to:} & \quad f(\theta,\tilde{\delta}_i^{(j)}) \leq \xi_i, \ \,\,j=1,\ldots, M; \;  i = 1, \ldots N, \nonumber
	\end{align}
	where $c(\theta):\Theta \to \R$ is a convex cost functional, $f(\theta,\delta):\Theta\times\Delta \to \R$ is convex in $\theta$ for any $\delta$, and $\{ \tilde{\delta}_i^{(1)},\ldots, \tilde{\delta}_i^{(M)} \} = \widehat{A}_{\delta_i}$ for all $i$. Owing to the $\xi_i$'s, problem \eqref{opt-relax} is always feasible, and it is assumed that it admits at least one minimizer for every $N$ and every ${\cal D}$ in its feasibility domain. When the minimizer is not unique, $\theta^\ast_{\widehat{A}}$ is singled out by selecting among the minimizers the one that further minimizes a tie-breaking convex functional $t_1(\theta)$ and, possibly, other convex functionals $t_2(\theta), t_3(\theta),\ldots$ in succession if the tie still occurs. Functional $f(\theta,\delta)$ is meant to quantify the \emph{level of appropriateness} of hypothesis $\theta$ for a given $\delta$ (refer to the SVR example where $f(\theta,\delta)$ is the vertical displacement between $y$ and the value of the linear model corresponding to $u$, to which the value $\gamma$ is subtracted). We say that a hypothesis $\theta$ is \emph{inappropriate} for $\delta$ if $f(\theta,\delta) > 0$ (in SVR, this corresponds to $y$ being away from $(w)^\top u+b$ more than $\gamma$). Variables $\xi_i$ are used to relax the constraint that the selected $\theta$ is appropriate for all $\tilde{\delta}_i^{(j)}$'s, and $\rho \cdot \xi_i$ is a penalty paid for inappropriateness. The hyper-parameter $\rho$ is used to tune the penalty so as to express more or less regret in case of constraint violation. In the special case in which $\widehat{A}_{\delta} = \{ \delta \}$ for all $\delta \in \Delta$, one goes back to a standard (non-adversarial) learning scheme. See also the next Section \ref{sec:opt-relax-apps} for more discussion on the interpretation of~\eqref{opt-relax}. \\
	
	In the present context the notion of adversarial risk becomes as follows.
	\begin{definition}[Adversarial inappropriateness and Adversarial risk] 
		A hypothesis $\theta$ is adversarially inappropriate for $\delta$ if there exists a $\tilde{\delta} \in A_{\delta}$ such that $f(\theta,\tilde{\delta}) > 0$. \\
		The {\em adversarial risk} of a hypothesis $\theta$, denoted $\textnormal{Risk}_A(\theta)$, is the probability of adversarial inappropriateness, i.e.,
		$$
		\textnormal{Risk}_A(\theta):=\prob \{ \delta:  \exists \tilde{\delta} \in A_{\delta} \mbox{ such that } f(\theta,\tilde{\delta}) > 0 \}.
		$$
		\hfill$\star$
	\end{definition}
	\noindent
	When $A_\delta = \{ \delta \}$, we simply speak of ``inappropriateness'' and the adversarial risk becomes the ``risk'' according to the following definition: $\textnormal{Risk}(\theta) := \prob \{ \delta:  \; f(\theta,\delta) > 0 \}$. \\
	
	The main thrust of this section is that the adversarial risk of hypothesis $\theta^\ast_{\widehat{A}}$ obtained by solving \eqref{opt-relax} can be accurately estimated from an observable quantity, which we again call ``adversarial complexity'' since it generalizes the same notion given in Section \ref{sec:SVR}, Definition \ref{def:adv_complex_outer}, for SVR.
	\begin{definition}[Adversarial complexity -- relaxed optimization schemes] \label{def:adv-cmplx-general}
		The adversarial complexity of $\theta^\ast_{\widehat{A}}$, denoted $s_{{A},\widehat{A}}^\ast$, is the number of data points $\delta_i$ that satisfy at least one of the following three conditions  
		\begin{itemize}
			\item[(i)] $f(\theta^\ast_{\widehat{A}},\tilde{\delta}_i) > 0$ for at least one $\tilde{\delta}_i \in A_{\delta_i}$
			\item[(ii)] $f(\theta^\ast_{\widehat{A}},\tilde{\delta}_i^{(j)}) = 0$ for at least one $\tilde{\delta}_i^{(j)} \in \widehat{A}_{\delta_i}$ 
			\item[(iii)]  $f(\theta^\ast_{\widehat{A}},\tilde{\delta}_i^{(j)}) > 0$ for at least one $\tilde{\delta}_i^{(j)} \in \widehat{A}_{\delta_i}$.
		\end{itemize}\hfill$\star$
	\end{definition}
	The only assumption we need to prove our results is the following mild condition of non-accumulation of $f(\theta,\delta)$ (this assumption replaces Assumption \ref{non-acc} for SVR. Indeed, the reader can verify that in the proof of the result for SVR -- more specifically, in the proof of Proposition \ref{prop:non-acc} in Section \ref{sec:derivations} -- Assumption \ref{non-acc} serves the only purpose of ensuring that Assumption \ref{non-acc-general} holds true). 
	\begin{assumption} \label{non-acc-general}
		For every $\theta$, it holds that 
		$$
		\prob \left\{\delta : \; \exists \, \tilde{\delta}^{(j)} \in \widehat{A}_\delta \text{ such that } f(\theta,\tilde{\delta}^{(j)}) = 0 \right\} = 0.
		$$ 	\hfill $\star$
	\end{assumption}
	We are now ready to state the main results of this section, Theorems \ref{th:general1} and \ref{th:general2}. These theorems are the counterparts within the current general setup of Theorems \ref{main_th_standard} and \ref{main_th_outer}. For more explanation and interpretation of the results, the reader is referred to Section \ref{sec:SVR}, as the discussion provided there can be easily adapted to Theorems \ref{th:general1} and \ref{th:general2}. 
	\begin{theorem} \label{th:general1}
		Under Assumption \ref{non-acc-general} and the condition that $\widehat{A}_\delta \subseteq A_\delta$ for all $\delta \in \Delta$, it holds that 
		\begin{equation} 
			\prob^N \{\,{\cal D} \,:\, \underline{\eps}(s_{A,\widehat{A}}^\ast)\leq \textnormal{Risk}_A(\theta^\ast_{\widehat{A}})\leq \overline{\eps}(s_{A,\widehat{A}}^\ast)\,\,\}\,\geq\,1-\beta,
		\end{equation}
		where $\theta^\ast_{\widehat{A}}$ is the hypothesis obtained from \eqref{opt-relax} and $s_{A,\widehat{A}}^\ast$ is its adversarial complexity according to Definition \ref{def:adv-cmplx-general}. \hfill$\star$
	\end{theorem}
	
	\begin{proof}
		See Section \ref{sec:derivations-extension}.
	\end{proof}
	
	\begin{theorem}	\label{th:general2}
		Under the sole Assumption \ref{non-acc-general} (without the requirement that $\widehat{A}_\delta \subseteq A_\delta$), it holds that 
		\begin{equation} 
			\prob^N \{\,{\cal D} \,:\,  \textnormal{Risk}_A(\theta^\ast_{\widehat{A}})\leq \overline{\eps}(s_{A,\widehat{A}}^\ast)\,\,\}\,\geq\,1-\beta,
		\end{equation}
		where $\theta^\ast_{\widehat{A}}$ is the hypothesis obtained from \eqref{opt-relax} and $s_{A,\widehat{A}}^\ast$ is its adversarial complexity according to Definition \ref{def:adv-cmplx-general}. \hfill$\star$
	\end{theorem}
	
	\begin{proof}
		See Section \ref{sec:derivations-extension}.
	\end{proof}
	
	\subsection{Some domains of application} 
	\label{sec:opt-relax-apps}
	
	Our goal in this paper was to present and discuss our results for SVR, followed by a formal proof that they extend to the general setup of relaxed optimization, while leaving the details of this extension's utilization to future contributions (since the present paper is already long and dense in its current form). Nevertheless, we find it advisable to at least briefly touch upon here some potential directions for its use. \\
	
	Optimization with constraint relaxation lies at the very core of all Support Vector (SV) methods. This includes:
	\begin{itemize}
		\item[-] all variants of SVR, namely SVR with fixed size, \cite{SmolaScholkopf2004}, and SVR with width depending on $u$ (these regression models are also known as Interval Predictor Models (IPM)
		and have been studied in  \cite{CaCaGa:09,CreKenGie2015,crespo2016interval,GarCamCare2019}). In these prediction schemes, the satisfaction of Assumption \ref{non-acc-general} can be secured by conditions akin to Assumption \ref{non-acc}.
		\item[-] SV methods for novelty/outlier detection, like e.g. one-class SVM, \cite{ScWiSmSTPl:99}, Support Vector Data-Description (SVDD), \cite{TaxDuin2004,WangChungWang2011}, and methods based on sliced-normal distributions, \cite{CreColKenGie2019}. In this setup, data points are vectors that contain features describing the members of a given population and the objective is to construct a descriptive region (e.g., in SVDD, this is a ball in a lifted feature space) that covers a high portion of the population distribution. 
		In this case, the satisfaction of Assumption \ref{non-acc-general} follows from requiring that the distribution of the population does not accumulate anomalously, e.g. that it admits a density. 
		\item[-] the framework of optimization with constraint relaxation is also in use in Support Vector Machines (SVM) for classification problems, \cite{CortesVapnik1995,VapnikBOOK}. It is fair to notice, however, that the circumstance that $y$ is a label taking value from a finite alphabet (e.g., from $\{0,1\}$ in binary classification) makes it more difficult to secure the satisfaction of Assumption \ref{non-acc-general} in this context. The reasons of this fact are discussed in \cite{Campi21ml} in a non-adversarial setup. We envisage that 
		the discussion in \cite{Campi21ml} can be carried over to the present adversarial setting, and in particular that the argument used in \cite{Campi21ml} to circumvent the problem in a non-adversarial setup can also be adopted in the adversarial setting. 
	\end{itemize}   
	
	Interestingly, theoretical results similar to those valid for \eqref{opt-relax} are expected to be usable in classification based on \emph{empirical error minimization}. To this end, consider any family of classifiers $Y_\theta(\cdot)$, where each value of $\theta$ defines a map from the instance space of a variable $u$ to a label, say, $y \in \{0,1\}$. Setting $c(\theta) = 0$, $\rho = 1$ and $f(\theta,u) = \One{Y_\theta(u) \neq y}$ (where $\One{\cdot}$ denotes the indicator function), problem \eqref{opt-relax} becomes 
	$$
	\min_{\theta \in \Theta} \sum_{i=1}^{N} \One{Y_\theta(\tilde{u}_i^{(j)}) \neq \tilde{y}_i^{(j)} \text{ for at least a } j\in\{1,\ldots,M\} },
	$$
	which corresponds to an adversarial empirical error minimization over the training set. The difficulty with this setup rests in the fact that it is not convex, as required in \eqref{opt-relax}. However, scenario results underpinning the achievements of this paper have been recently extended to a non-convex setup in the foundational work \cite{GarCam2024}, and we expect these new results to be carried over so as to cover classification via empirical error minimization. \\
	
	In addition, we would like to point out that optimization with constraint relaxation has been also used within the context of the so-called \emph{scenario approach}, \cite{CamGaBOOK2018}, a flexible scheme for data-driven decision-making, \cite{campi2021Annual,Garatti22}. In this context, $\theta$ represents a decision (e.g., the parameter of a controller, or a portfolio in an investment problem) rather than a model, and parameter $\delta$ indicates a realization of the environment to which the decision is applied (e.g., the transfer function of the system to be controlled, or the evolution of the market in an investment problem). 
	While the results of this section open new perspectives toward establishing a new adversarial theory for scenario data-driven decision-making, a complete discussion is beyond the scope of the present paper.

	\section{Risk evaluations for out-of-distribution observations} 
	\label{sec:ambiguous} 
	
	\newcommand{\probq}{{\mathbb Q}}
	
	The findings of Section \ref{sec:opt-relax} carry significant implications for addressing (non-adversarial) risk evaluations in problems where the training set ${\cal D} = \{\delta_i\}_{i=1}^N$, $i=1,\ldots,N$, is an i.i.d.\ sample drawn according to a probability $\prob$ and one wants to provide risk evaluations for new observations coming from a  different probability $\prob'$. As an example, the training set ${\cal D}$ may come from a laboratory environment (i.e., a {\em simulator}), and the $\delta'$ against which the hypothesis is used comes from the real world. 
	This problem falls within the field of \emph{out-of-distribution} generalization theory, a topic of growing importance in the machine learning community, \cite{ErdIyen2026,Domain_Generalization_2023,NEURIPS2021_c5c1cb0b,Wang_etal_2023,liu2023outofdistribution}. Our results share similarities with the recent work \cite{YanPariseBitar2022}; however, by leveraging the new adversarial results presented in Section \ref{sec:opt-relax}, we can adopt a more general perspective than \cite{YanPariseBitar2022}, as discussed following the statement of Theorem \ref{th:ambiguous}. 
	\\ 
	
	In the following, we assume that both $\prob$ and $\prob'$ are unknown. On the other hand, as is clear, keeping control on the risk associated with observations coming from $\prob'$, while only having access to a dataset drawn from $\prob$, requires introducing some information on the mismatch between the two; to this end, we use the well-known \emph{Wasserstein metric}.\footnote{The Wasserstein metric is a flexible tool, popular in many fields, also largely adopted in the emerging area of Distributionally Robust Optimization (DRO), \cite{PeymanKuhn2018,Kuhnetal2019,GaoChenKleywegt2022}.} Start by assuming that $\Delta$ is a metric space with distance $d(\cdot, \cdot)$. No restrictions on $\Delta$ and $d(\cdot,\cdot)$ are introduced. The Wasserstein distance of $\prob$ and $\prob'$ is defined as 
	\begin{displaymath}
		{\mathcal W}(\prob, \prob') := \inf_{\probq}\ \E_\probq \left[ d(\delta, \delta') \right],
	\end{displaymath}
	where $(\delta, \delta')$ is a random element from $(\Delta\times\Delta, \mathcal{G}\otimes\mathcal{G}, \probq)$,
	and the infimum is taken over all probabilities $\probq$ on $(\Delta\times\Delta, \mathcal{G}\otimes\mathcal{G})$
	whose first and second marginals are, respectively, $\prob$ and $\prob'$.\footnote{In more explicit terms: for all $G \in \mathcal{G}$, $\probq\{(\delta, \delta') :\ \delta\in G\} = \prob\{G\}$ and $\probq\{(\delta, \delta') :\ \delta'\in G\} = \prob'\{G\}$.} 
	%
	\\
	
	The following assumption coincides with that in \cite{YanPariseBitar2022}.
	\begin{assumption} \label{assumption-ambiguous}
		$ $
		\begin{displaymath}
			{\mathcal W}(\prob, \prob') \leq \mu, \quad\text{\rm for some $\mu >0$, known to the user}.
		\end{displaymath}
		\hfill $\star$
	\end{assumption}
	
	\noindent
	
	We mean to study the out-of-distribution risk of $\theta^\ast_{\widehat{A}}$, where the out-of-distribution risk for a $\theta \in \Theta$ is defined as follows. 
	\begin{definition}[Out-of-distribution risk] 
		$$
		\textnormal{Risk}'(\theta) := \prob' \{ \delta':  f(\theta,\delta') > 0 \}.
		$$
		\hfill$\star$
	\end{definition}
	
	To conduct this study by resorting to the adversarial results of Section \ref{sec:opt-relax} as a tool of investigation,\footnote{We repeat that the evaluations we want to carry out in this section refer to the standard setup with non-adversarial actions; in this endeavor, adversarial results are used as an enabling tool of investigation.} consider adversarial regions $A_\delta$ that are {\em closed balls} in $\Delta$: $A_\delta = \{\tilde{\delta}\ :\ d(\tilde{\delta}, \delta) \leq R\}$ for some $R \geq 0$ (think of $R$ as a free parameter that can be tuned when pursuing the evaluation of the out-of-distribution risk). 
	$\widehat{A}_\delta$ instead is completely free, it can be any finite set of points in $\Delta$ that varies with $\delta$, and it may well be that $\widehat{A}_\delta = \{ \delta \}$.
	$\theta^\ast_{\widehat{A}}$ is the hypothesis obtained from \eqref{opt-relax}, and $s_{A,\widehat{A}}^\ast$ its adversarial complexity. 
	Note that only the adversarial complexity $s_{A,\widehat{A}}^\ast$ depends on $A_{\delta}$, and hence on $R$, while the hypothesis $\theta^\ast_{\widehat{A}}$ does not. 
	%
	The main result of this section, Theorem \ref{th:ambiguous}, shows that the adversarial complexity $s_{A,\widehat{A}}^\ast$, along with the knowledge of $\mu$, allows one to evaluate the out-of-distribution risk. 
	
	\begin{theorem}
		\label{th:ambiguous}
		Under Assumptions \ref{non-acc-general} and \ref{assumption-ambiguous}, it holds that
		\begin{equation} 
			\label{result out-of-distribution}
			\prob^N \left\lbrace {\cal D} \ :\  \mathrm{Risk}'(\theta^\ast_{\widehat{A}})\ \leq\ \overline{\eps}(s_{A,\widehat{A}}^\ast) + \frac{\mu}{R} \right\rbrace \geq 1-\beta.
		\end{equation}
		\hfill $\star$
	\end{theorem}
	
	\begin{proof}
		By the Wasserstein bound in Assumption \ref{assumption-ambiguous}, and by the definition of infimum, for all $\eta>0$ there exists a probability
		$\probq$, with marginals $\prob$ and $\prob'$, such that
		$\E_\probq \left[ d(\delta, \delta') \right] \leq \mu + \eta$.
		By Markov's inequality, for this $\probq$ it holds that 
		\begin{equation}
			\label{eq:amb1}
			\probq \{(\delta, \delta'):\  d(\delta, \delta') > R \}
			\leq 
			\frac{\E_\probq \left[ d(\delta, \delta') \right]}{R} \leq \frac{\mu + \eta}{R}.
		\end{equation}
		Recall that $A_\delta = \{\tilde{\delta} :\ d(\tilde{\delta}, \delta) \leq R\}$. For any $\theta$, we have 
		\begin{align*}
			&\{(\delta, \delta') :\ f(\theta,\delta') > 0 \} \\
			&\quad = \{(\delta, \delta') :\ d(\delta, \delta')\leq R \ \land\ f(\theta,\delta') > 0 \} \ \cup \ \{(\delta, \delta') :\ d(\delta, \delta') > R \ \land\ f(\theta,\delta')>0 \} \\
			&\quad \subseteq \{(\delta, \delta') :\ \exists \tilde{\delta}\ \text{with} \ d(\delta, \tilde{\delta}) \leq R \ \land\ f(\theta,\tilde{\delta}) > 0 \} \ \cup \ \{(\delta, \delta') :\ d(\delta, \delta') > R \} \\
			&\quad = \{(\delta, \delta') :\ \exists \tilde{\delta}\in A_\delta \ \text{\rm s.t.}\ f(\theta,\tilde{\delta}) > 0 \} \ \cup \ \{(\delta, \delta') :\ d(\delta, \delta') > R \}, 
		\end{align*}
		from which, by sub-additivity and \eqref{eq:amb1}, we obtain 
		\begin{align*}
			\mathrm{Risk}'(\theta) 
			&= \prob'\{\delta' :\ f(\theta,\delta') > 0 \} \\
			&= \probq \{(\delta, \delta') :\ f(\theta,\delta') > 0 \} \\
			& \leq \probq\{(\delta, \delta') :\ \exists \tilde{\delta}\in A_\delta \ \text{\rm s.t.}\ f(\theta,\tilde{\delta}) > 0 \} \ + \ \probq\{(\delta, \delta')\ :\ d(\delta, \delta') > R \} \\
			& \leq \prob\{\delta :\ \exists \tilde{\delta}\in A_\delta \ \text{\rm s.t.}\ f(\theta,\tilde{\delta}) > 0 \} \ + \ \frac{\mu + \eta}{R} \\
			& = \textnormal{Risk}_A(\theta) \ + \ \frac{\mu + \eta}{R}.
		\end{align*}
		Since this result is true for every $\eta > 0$, it follows that
		\begin{equation}
			\label{Risk'<=}
			\mathrm{Risk}'(\theta) 
			\leq \mathrm{Risk}_A(\theta) \ + \ \frac{\mu}{R}. 
		\end{equation} 
		This implies that 
		\begin{align*}
			\left\lbrace {\cal D}:\  \mathrm{Risk}_A(\theta^\ast_{\widehat{A}})\ \leq\ \overline{\eps}(s_{A,\widehat{A}}^\ast) \right\rbrace
			\subseteq
			\left\lbrace {\cal D}:\  \mathrm{Risk}'(\theta^\ast_{\widehat{A}})\ \leq\ \overline{\eps}(s_{A,\widehat{A}}^\ast) + \frac{\mu}{R} \right\rbrace,
		\end{align*}
		which gives 
		\begin{eqnarray*}
			\prob^N \left\lbrace {\cal D}:\  \mathrm{Risk}'(\theta^\ast_{\widehat{A}})\ \leq\ \overline{\eps}(s_{A,\widehat{A}}^\ast) + \frac{\mu}{R} \right \rbrace & \geq &
			\prob^N \left\lbrace {\cal D}:\  \mathrm{Risk}_A(\theta^\ast_{\widehat{A}})\ \leq\ \overline{\eps}(s_{A,\widehat{A}}^\ast) \right\rbrace \\
			& \geq & 1-\beta,
		\end{eqnarray*}
		where the last inequality follows from Theorem \ref{th:general2}.
	\end{proof}
	
	The first part of the proof of Theorem \ref{th:ambiguous} closely follows an argument used in \cite[Lemma 1]{YanPariseBitar2022}, which was also used in \cite{Luedtke_Ahmed_2008} and \cite{FALSONE2019108537} in a different context. The main contribution compared to \cite{YanPariseBitar2022} lies in utilizing Theorem \ref{th:general2} to bound $\mathrm{Risk}_A(\theta^\ast_{\widehat{A}})$ in the last part of the proof, leading to a significantly stronger result than that in \cite[Theorem 3]{YanPariseBitar2022} in two respects:
	\begin{itemize}
		\item[i.] differently from \cite{YanPariseBitar2022}, our bound on $\mathrm{Risk}'(\theta^\ast_{\widehat{A}})$ is adapted to the complexity $s^\ast_{A,\widehat{A}}$, a statistic of the data, which enables tracking the actual value of $\mathrm{Risk}'(\theta^\ast_{\widehat{A}})$ from one experiment to another without introducing over-conservatism;
		\item[ii.] thanks to the introduction of the advanced notion of adversarial complexity, our result can be applied to solutions that are decoupled from the assumed Wasserstein distance between $\prob$ and $\prob'$. In particular, Theorem \ref{th:ambiguous} applies when $\widehat{A}_\delta = \{ \delta \}$, i.e., when $\theta^\ast_{\widehat{A}}$ is just a standard, non-robust, solution. This is different from \cite{YanPariseBitar2022}, whose main result is only applicable to solutions satisfying the infinitely many constraints $f(\theta,\tilde{\delta}) \leq 0$, $\forall \tilde{\delta} \in A_{\delta_i}$, $i=1,\ldots,N$, where $A_{\delta_i}$ is tuned to the Wasserstein bound. 
	\end{itemize}
	As previously noted, $R$ plays the role of a tunable parameter, and the result in Theorem~\ref{th:ambiguous} holds for any choice of the value of $R$. As a consequence, the user can play with $R$ to optimize the bound on $\mathrm{Risk}'(\theta^\ast_{\widehat{A}})$ given in Theorem \ref{th:ambiguous}. As $R$ increases, $s^\ast_{A,\widehat{A}}$ (and, thereby, $\overline{\eps}(s^\ast_{A,\widehat{A}})$) tends to increase while $\mu/R$ diminishes. While the best compromise is difficult to foresee, one can experimentally try various choices $R_1 < R_2 < \cdots < R_i <  \cdots R_h$ and select the one giving the best result. The corresponding confidence level can be bounded as follows: 
	%
	\begin{align*}
		& \prob^N \left\lbrace {\cal D}:\  \textnormal{Risk}'(\theta^\ast_{\widehat{A}})\ >\ \overline{\eps}(s_{A,\widehat{A},i}^\ast) + \frac{\mu}{R_i} \ \ \text{\rm for at least one}\ i \in \{1,\ldots h\} \right\rbrace  \\
		& \quad\quad \leq \sum_{i=1}^h \prob^N \left\lbrace {\cal D}:\  \textnormal{Risk}'(\theta^\ast_{\widehat{A}})\ >\ \overline{\eps}(s_{A,\widehat{A},i}^\ast) + \frac{\mu}{R_i} \right\rbrace \\
		& \quad\quad \leq \sum_{i=1}^h \beta = h\beta,
	\end{align*}
	from which 
	\begin{equation}
		\label{confidence-h-cases}
		\prob^N \left\lbrace {\cal D}:\  \textnormal{Risk}'(\theta^\ast_{\widehat{A}})\ \leq\ \overline{\eps}(s_{A,\widehat{A},i}^\ast) + \frac{\mu}{R_i} \ \ \text{\rm for all}\ i=1,\ldots h \right\rbrace \geq 1 - h\beta.
	\end{equation}
	%
	Therefore, the user can claim the result obtained by minimizing over the tested choices of $R$ with confidence $1- h\beta$. The presence of $h$ in front of $\beta$ has quite a minor impact because the dependence of $\overline{\eps}$ on $1/\beta$ is logarithmic (see Section \ref{sec:bounding_adv_risk}), which implies that $\beta$ can be made quite small without significantly affecting the upper bound on the risk. 
	%
	
	\begin{remark}[about ``$\sup_{\prob'}$'']\label{remark sup P'} Theorem \ref{th:ambiguous} states the result \eqref{result out-of-distribution}, which holds for any $\prob'$ belonging to a Wasserstein ball of radius $\mu$ centered in $\prob$. Therefore, equation \eqref{result out-of-distribution} might also be expressed by adding a ``$\sup_{\prob'}$'' in front of its left-hand side (in notation $\sup_{\prob'}$ we have omitted for brevity the specification that $\prob'$ belongs to the Wasserstein ball). Interestingly, we may show that the result also holds in a somewhat stronger sense. Start by considering equation \eqref{Risk'<=}; it can be re-written as $\sup_{\prob'} \mathrm{Risk}'(\theta) \leq \sup_{\prob'} \left[ \mathrm{Risk}_A(\theta) \ + \ \frac{\mu}{R} \right] = \mathrm{Risk}_A(\theta) \ + \ \frac{\mu}{R}$, where $\sup_{\prob'}$ has been suppressed in the last step because the right-hand side does not depend on $\prob'$. Consequently, the two equations in displaymath that follow \eqref{Risk'<=} can also be re-written as $\left\lbrace {\cal D}:\  \mathrm{Risk}_A(\theta^\ast_{\widehat{A}})\ \leq\ \overline{\eps}(s_{A,\widehat{A}}^\ast) \right\rbrace \subseteq \left\lbrace {\cal D}:\  \sup_{\prob'}\mathrm{Risk}'(\theta^\ast_{\widehat{A}})\ \leq\ \overline{\eps}(s_{A,\widehat{A}}^\ast) + \frac{\mu}{R} \right\rbrace$ and $ \prob^N \left\lbrace {\cal D}:\ \sup_{\prob'}\mathrm{Risk}'(\theta^\ast_{\widehat{A}})\ \leq\ \overline{\eps}(s_{A,\widehat{A}}^\ast) + \frac{\mu}{R} \right \rbrace \geq \prob^N \left\lbrace {\cal D}:\  \mathrm{Risk}_A(\theta^\ast_{\widehat{A}})\ \leq\ \overline{\eps}(s_{A,\widehat{A}}^\ast) \right\rbrace \geq1-\beta$. In this last result, ``$\sup_{\prob'}$'' appears in front of $\mathrm{Risk}'(\theta^\ast_{\widehat{A}})$, showing that the bound on the risk continues to hold when the choice of $\prob'$ is ``adapted'' to the construction of $\theta^\ast_{\widehat{A}}$ based on $\cal D$: the probability of drawing a sample $\cal D$ for which there exists an out-of-sample distribution $\prob'$ leading to a $\mathrm{Risk}'(\theta^\ast_{\widehat{A}})$ exceeding $\overline{\eps}(s_{A,\widehat{A}}^\ast) + \frac{\mu}{R}$ is no more than $\beta$. \hfill$\star$ 	
	\end{remark} 
	
	\begin{example}[convex hull] 
		\label{convex hull-example}
		\begin{figure}[h!]
			\centering  \includegraphics[width=7cm]{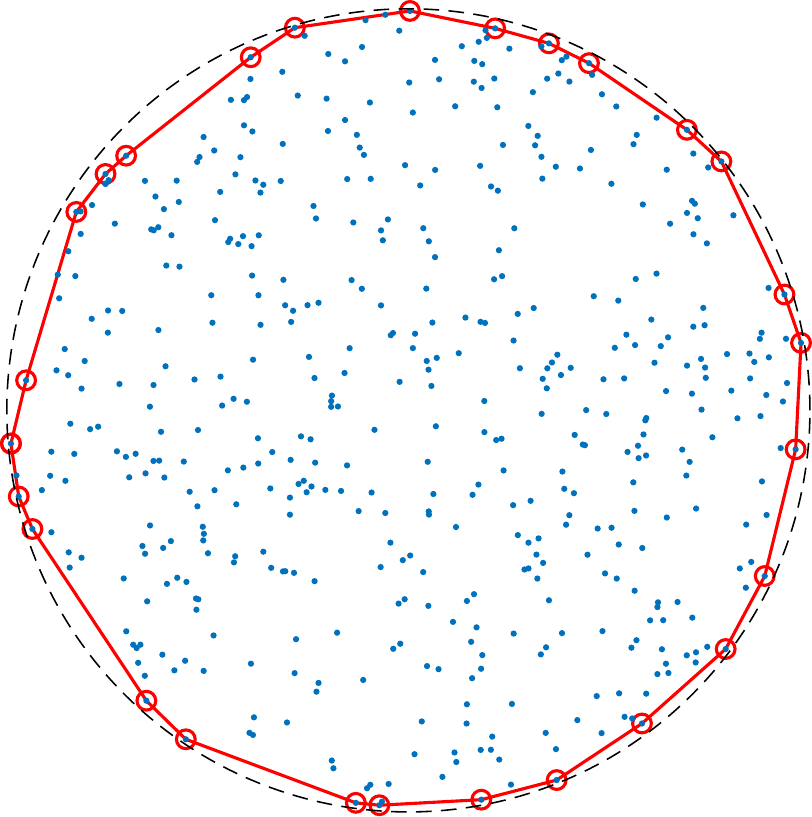}
			\caption{Convex hull of $500$ points in ${\mathbb R}^{2}$.}
			\label{convex hull}
		\end{figure}
		We provide a numerical example to better illustrate the results of this section. $N = 500$ points are drawn i.i.d. from a unitary-radius disk in ${\mathbb R}^{2}$ with a uniform distribution $\prob$, and their \emph{convex-hull}, i.e., the smallest convex set that contains all points, is constructed (see Figure \ref{convex hull}). In this context, we identify the $\delta_i$'s with the points, while a $\theta$ represents a closed convex set in ${\mathbb R}^{2}$. Constructing the convex-hull amounts to solve a problem in the form \eqref{opt-relax}, with $\rho$ large enough, where $\hat{A}_{\delta_i} = \{ \delta_i \}$, and function $f(\theta,\delta)$ is zero when the point $\delta$ is in the set $\theta$ and takes a value that grows linearly with the distance between the point and the convex set when the point is outside.\footnote{See Appendix \ref{appendix convex hull} for a complete formalization of how this problem is framed within the framework of~\eqref{opt-relax}.} \\
		
		Theorem \ref{th:ambiguous} is used to upper bound the out-of-distribution risk of the convex hull (i.e., the probability that a new point lies outside the convex-hull) when the Wasserstein budget is $\mu = 10^{-3}$. We consider $30$ possible choices of $R$, namely $R_i = \mu+2\mu(i-1)$, $i = 1, \ldots, 30$, and set $\beta$ to the value $10^{-3} /30$, which, according to \eqref{confidence-h-cases}, corresponds to a confidence value of $1 - 30\beta = 1 - 10^{-3}$. Figure \ref{out-of-distribution bound} shows the result. 
		\begin{figure}[h!]
			\centering  \includegraphics[width=9cm]{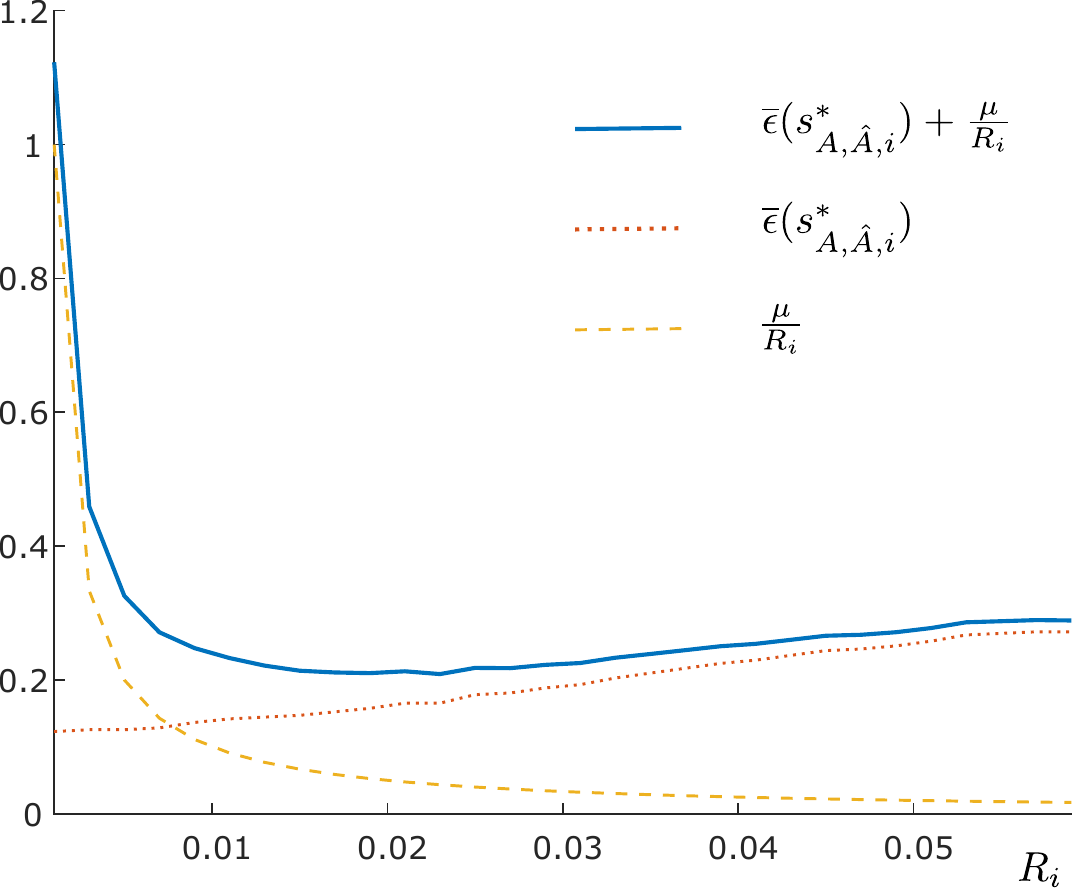}
			\caption{Upper bound (blue solid profile) to the out-of-distribution risk for $30$ values of $R$ as shown in the abscissa ($N = 500$, confidence $= 1 - 10^{-3}$). The bound is formed by two components having opposite trend as $R_i$ increases.}
			\label{out-of-distribution bound}
		\end{figure}
		The minimum is attained for $R_{12}$ with value of the bound equal to $0.2088$. We further compare this bound with the actual out-of-distribution risk obtained in two cases. \\
		
		The first case corresponds to constructing $\prob'$ by moving to the boundary of the disk the mass of $\prob$ that lies within the annulus whose outer boundary is the boundary of the disk and inner boundary selected so as to spend all Wasserstein budget. In other words, the annulus is emptied, and all the probabilistic mass within it is moved to the boundary of the disk. Figure \ref{emptied region} shows the emptied annulus. 
		\begin{figure}[h!]
			\centering  \includegraphics[width=7cm]{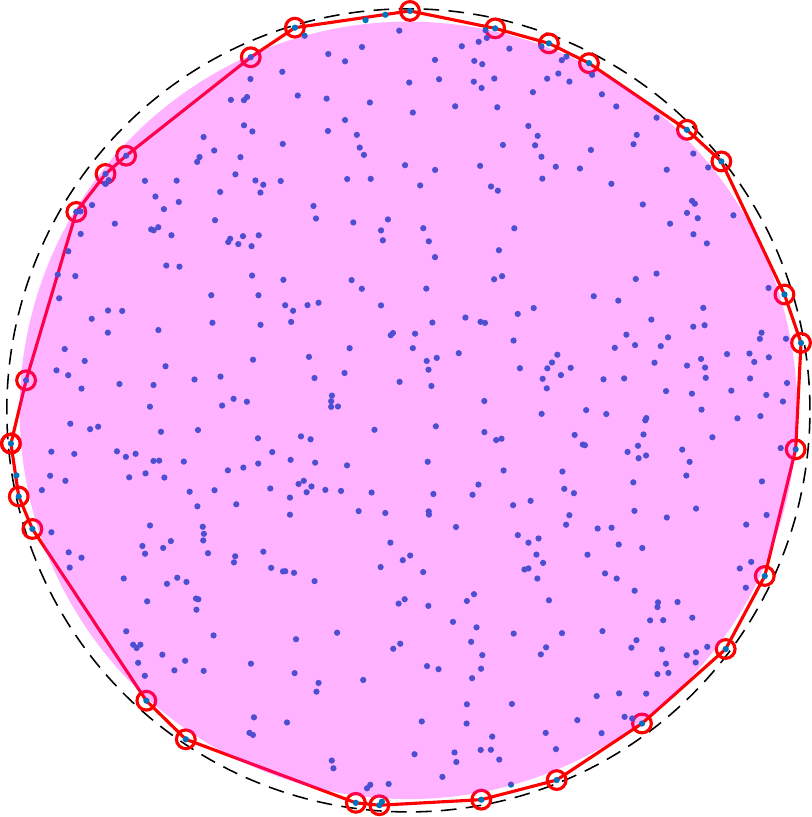}
			\caption{Convex hull and original disk (dashed line). The emptied annulus corresponds to the white peripheral portion of the disk.}
			\label{emptied region}
		\end{figure}
		Since, after this shift, all the moved mass certainly lies outside the convex hull, this $\prob'$ corresponds to a ``bad'' case (though not the worst, which is difficult to precisely envisage). The ensuing risk was calculated to be $0.0719$. \\
		
		An inspection of Figure \ref{emptied region} shows that a large quantity of the shifted mass corresponds to portions of the disk that already lied outside the convex hull, so the corresponding budget is spent fruitlessly. In the attempt to get closer to the worst case, we therefore conceive to only move (along a radial direction) the probabilistic mass in the peripheral part of the convex hull to a location just outside its boundary. This leads to the emptied region shown in Figure \ref{emptied region opt}. 
		\begin{figure}[h!]
			\centering  \includegraphics[width=7cm]{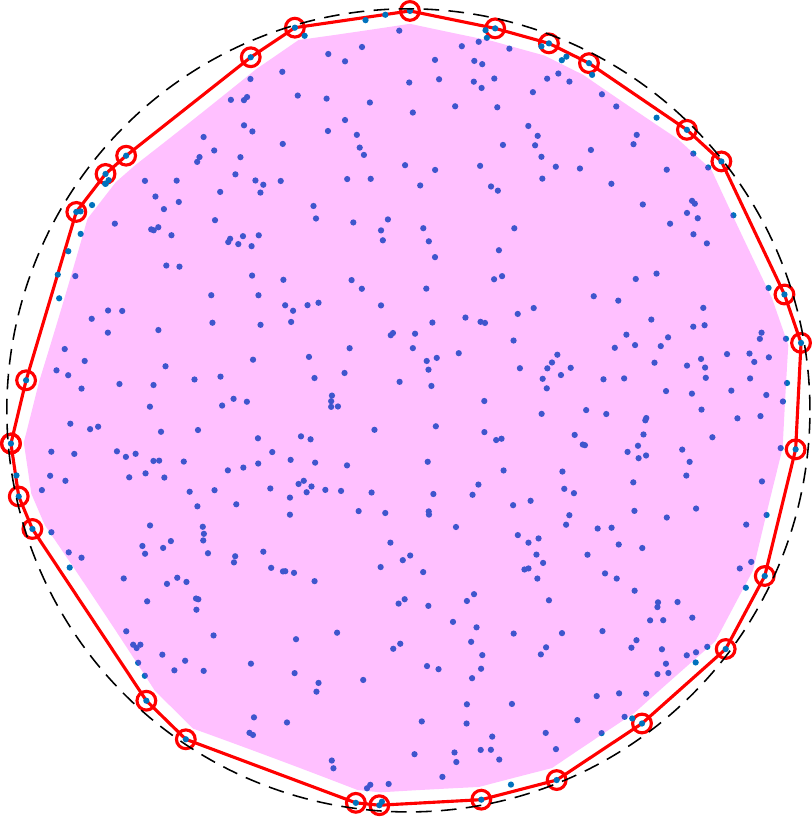}
			\caption{Emptied region obtained by moving radially only the probabilistic mass close to the boundary of the convex hull.}
			\label{emptied region opt}
		\end{figure}
		Note that this case corresponds to selecting a $\prob'$ adapted to the constructed convex hull, which is a valid choice as explained in Remark \ref{remark sup P'}. The ensuing risk turns out to be $0.1178$. In this case, the ratio of the bound of $0.2088$ to the actual risk is below the value of $2$. To appreciate the quality of this result, one should recall that the bound must hold for any $\prob$, while here we have just considered one $\prob$, i.e., the uniform probability distribution, and, moreover, the bound is enforced to hold with high confidence $1 - 10^{-3}$, while here we have just considered one single realization of the $500$ points. \\
		
		In a second experiment, we consider the same setup as described above but change the number of points, which is now $N = 2000$, as well as the Wasserstein budget, which is set to value $10^{-4}$. Both changes lead to a lowering of the risk. Figure \ref{out-of-distribution bound2} shows the profile of the upper bound on the out-of-distribution risk for $50$ possible choices of $R$, namely $R_i = \mu+5\mu(i-1)$, $i = 1, \ldots, 50$, and $\beta = 10^{-4}/50$ (corresponding to a confidence value of $1 - 10^{-4}$). 
		\begin{figure}[h!]
			\centering  \includegraphics[width=9cm]{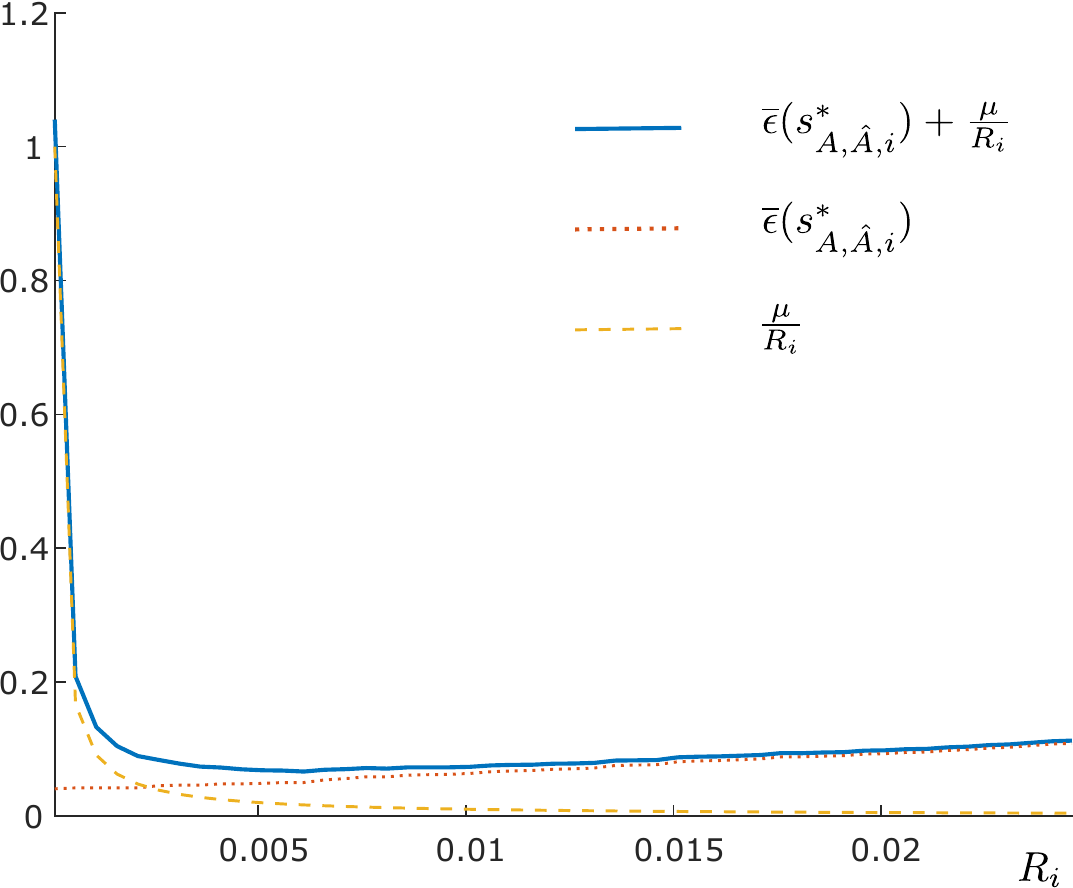}
			\caption{Upper bound (blue solid profile) to the out-of-distribution risk for $50$ values of $R$ as shown in the abscissa ($N = 2000$, confidence $= 1 - 10^{-4}$).}
			\label{out-of-distribution bound2}
		\end{figure} 
		In this case, the minimum is attained for $R_{13}$, with value of the bound equal to $0.0662$. The out-of-distribution risk obtained by moving the mass in the external annulus as described before is $0.0259$, while shifting the mass that lies in the proximity of the boundary of the convex hull gives an out-of-distribution risk of $0.0374$.  \hfill$\star$ 
	\end{example}
	
	\section{Proof of Theorems \ref{main_th_standard} and \ref{main_th_outer}} 
	\label{sec:derivations}
	
	\subsection{Some preliminary facts}
	
	{\bf Notations for SVR.} To make notations consistent with the section of extensions (Section \ref{sec:opt-relax}), also in the case of SVR we shall write $\delta$ in place of $(u,y)$, and $\Delta$ in place of $\mathbb{R}^d\times\mathbb{R}$. Similarly, $\delta_i$ stands for $(u_i,y_i)$, $\tilde{\delta}^{(j)}$ for $(\tilde{u}^{(j)},\tilde{y}^{(j)})
	$ and $\tilde{\delta}_i^{(j)}$ for $(\tilde{u}_i^{(j)},\tilde{y}_i^{(j)})$.
	Moreover, we let 
	$$
	c(\theta) = \gamma + \tau \| w \|^2 
	$$
	and also 
	\begin{equation} \label{def_of_f}
		f(\theta,\delta) = |y - w^\top u - b|-\gamma,
	\end{equation}
	so that the constraints in \eqref{svr}  
	$$
	|\tilde{y}_i^{(j)} - w^\top \tilde{u}_i^{(j)}  - b | - \gamma \leq \xi_i, \ \,\,j=1,\ldots, M; \;  i = 1, \ldots N.
	$$
	are re-written as 
	$$
	f(\theta,\tilde{\delta}_i^{(j)}) \leq \xi_i, \ \,\,j=1,\ldots, M; \;  i = 1, \ldots N. 
	$$
	\\
	{\bf Rationale behind the proof.} The main tool we shall use to establish Theorems \ref{main_th_standard} and \ref{main_th_outer} of this paper is Theorem 2 in reference \cite{Garatti22}. The study in \cite{Garatti22} is concerned with the characterization of the {\em probability of violation}, denoted $V(z^\ast_N)$,\footnote{In \cite{Garatti22}, the term ``risk'' is used to indicate $V(z^\ast_N)$; here, we shall call $V(z^\ast_N)$ the ``probability of violation'' because ``risk'' is used to indicate the risk of a SVR predictor.} of an abstract decision $z^\ast_N$ that is constructed according to an ``umbrella framework'' that encompasses various schemes as special cases. We shall see that the {\em adversarial risk} of the predictor ${\cal P}(\theta^\ast_{\widehat{A}})$ considered in this paper can be exactly related to a specific instance of $V(z^\ast_N)$. Nonetheless, tracing back the setup of the present paper to that of \cite{Garatti22} is nontrivial, and indeed the naive approach of simply identifying $z^\ast_N$ with $\theta^\ast_{\widehat{A}}$ neglects facts that play a central role in the analysis, and does not lead to any meaningful conclusions (this is because in \cite{Garatti22} the decision $z^\ast_N$ is required to satisfy certain conditions -- Assumptions 3 and 4 of \cite{Garatti22} -- that are not satisfied by $\theta^\ast_{\widehat{A}}$). As a consequence, we shall have to carefully introduce a more articulated definition of $z^\ast_N$. 
	\\
	
	A final notice goes to the fact that, to avoid annoying repetitions, the proofs of Theorems 1 and 2 will be carried out simultaneously, with just a quick distinction at the very end. 
	Correspondingly, 
	we refer to complexity as per Definition \ref{def:adv_complex_outer}, which preserves its validity both when $\widehat{A} \subseteq A$ and $\widehat{A} \not\subseteq A$. 
	
	\subsection{The proof}
	
	To set the stage, let $\calZ$ (called the space of decisions) be the set of pairs $(\theta,\calL)$, where $\theta \in \Theta$ and $\calL \in \mathcal{MS}$, where $\mathcal{MS}$ is the collection of all finite multisets of elements of $\Delta$.\footnote{A multiset is an unordered collections of elements that admits repetitions. Thus, for multisets we have e.g. that $\{a,a,b\} = \{a,b,a\} \neq \{a,b\}$.} To any $\delta \in \Delta$, we associate a subset of $\calZ$ defined as follows:
	\begin{equation} \label{eq:Z_delta}
		\calZ_\delta = \left\{ z = (\theta,\calL) \in \calZ: \; f(\theta,\tilde{\delta}) \leq 0, \forall \tilde{\delta} \in A_{\delta} \text{ and }  f(\theta,\tilde{\delta}^{(j)}) \leq 0, \forall \tilde{\delta}^{(j)} \in \widehat{A}_{\delta} \right\}
	\end{equation}
	(in more compact form, $\calZ_\delta = \{ z = (\theta,\calL) \in \calZ: \; A_\delta \cup \widehat{A}_\delta \subseteq \calP(\theta) \}$ -- note that the condition defining $\calZ_\delta$ does not involve the $\calL$ part of $z$). We have the following definition borrowed from  \cite{Garatti22}, Section 5. 
	
	\begin{definition}[Violation and Probability of violation] \label{def:violation}
		A decision $z \in \calZ$ is said to violate a $\delta \in \Delta$ when $z \notin \calZ_\delta$. The probability of violation of $z$ is defined as $V(z) := \prob \left\{\delta\in\Delta : \; z \notin \calZ_\delta \right\}$. \hfill$\star$
	\end{definition}
	\noindent Given the very definition of $\calZ_\delta$ in \eqref{eq:Z_delta}, the probability of violation of $z$ can be written more explicitly as
	$$
	V(z) = \prob \left\{\delta \in \Delta : \; A_{\delta} \cup \widehat{A}_{\delta} \not \subseteq \calP(\theta) \right\}. 
	$$
	The fact the $\calL$ component of $z$ plays no role in the concept of violation justifies expressions like ``$\theta$ violates $\delta$''. Moreover, this fact is key to maintain a connection with $\textnormal{Risk}_A(\theta) = \prob \left\{\delta \in \Delta : \; A_{\delta} \not \subseteq \calP(\theta) \right\}$, which is the quantity we are interested in. Since the following relation holds: $A_{\delta} \not \subseteq \calP(\theta) \Rightarrow  A_{\delta} \cup \widehat{A}_{\delta} \not \subseteq \calP(\theta)$, then we always have that $\textnormal{Risk}_A(\theta) \leq V(z)$. Moreover, when $\widehat{A} \subseteq A$, it holds that $A_{\delta} \cup \widehat{A}_{\delta} = A_{\delta}$, so that the stronger relation $\textnormal{Risk}_A(\theta) = V(z)$ holds. \\ 
	
	For any given $N$ and any sample of elements ${\cal D}=(\delta_1,\ldots,\delta_N)$ from $\Delta^N$, the data-driven decision $z^\ast_N$ is defined as the pair $(\theta^\ast_{\widehat{A}},\calL^\ast)$, where $\theta^\ast_{\widehat{A}}$ is the solution to \eqref{svr} (possibly singled out by a tie-break rule in case of multiple minimizers, as explained after \eqref{svr}), and $\calL^\ast$ is the multiset of the $\delta_i$, $i=1,\ldots,N$, that are violated by $\theta^\ast_{\widehat{A}}$, i.e., those for which 
	$A_{\delta_i} \cup \widehat{A}_{\delta_i} \not \subseteq \calP(\theta^\ast_{\widehat{A}})$
	. When $N = 0$, $z^\ast_0$ is formed by the unconstrained solution to \eqref{svr} and the empty multiset. The map from $\delta_1,\ldots,\delta_N$ to the decision $z^\ast_N$ is indicated by $M_N: \Delta^N \to \calZ$ and the notation $M_N(\delta_1,\ldots,\delta_N)$ is also in use to indicate $z^\ast_N$ when we want to specify the sample $\delta_1,\ldots,\delta_N$ that has generated the decision. \\  
	
	Before proceeding, we also need to recall  from \cite{Garatti22} the notion of {\em support element}: a $\delta_i$ in the sample $\delta_1,\ldots,\delta_N$ is called a {\em support element} if 
	$M_N(\delta_1,\ldots,\delta_N) \neq M_{N-1}(\delta_1,\ldots,\delta_{i-1}, \delta_{i+1},\ldots,\delta_N)$, i.e., removing $\delta_i$ from $\delta_1,\ldots,\delta_N$ changes the decision. \\
	
	The outline of the rest of the proof is as follows. We want to invoke Theorem 2 in \cite{Garatti22} to establish upper and lower bounds for $V(z^\ast_N)$ (from which, bounds for $\textnormal{Risk}_A(\theta)$ can be derived). To apply Theorem 2 in \cite{Garatti22}, we need to verify that the family of maps $M_N$, $N=0,1,\ldots$, satisfies the so-called consistency Assumption 3 of \cite{Garatti22} and the non-degeneracy Assumption 4 of \cite{Garatti22}. The satisfaction of these two assumptions is stated below as Lemma~\ref{main-prop:consistency} and Lemma~\ref{main-prop:nondegeneracy}, respectively. After proving these lemmas, the conclusion will be drawn by leveraging the connections between $\textnormal{Risk}_A(\theta^\ast_{\widehat{A}})$ and $V(z^\ast_N)$. 
	
	\begin{lemma}[Consistency of $M_N$]
		\label{main-prop:consistency} The family of maps $M_N$, $N=0,1,\ldots$, satisfies Assumption 3 of \cite{Garatti22}, namely, the following properties hold:
		\begin{itemize}
			\item[-] \textsc{Permutation invariance:} for every $N$ and every $(\delta_1,\ldots,\delta_N) \in \Delta^N$, given any permutation $(i_1,\ldots,i_N)$ of $(1,\ldots,N)$ it holds that $M_N(\delta_1,\ldots,\delta_N) = M_N(\delta_{i_1},\ldots,\delta_{i_N})$; 
			\item[-] \textsc{Stability in the case of confirmation:} for every integers $N_1$ and $N_2$, if $\delta_1, \ldots, \allowbreak \delta_{N_1}, \delta_{N_1+1}, \ldots, \delta_{N_1+N_2}$ are such that 
			$$
			M_{N_1}(\delta_1,\ldots,\delta_{N_1}) \in \calZ_{\delta_{N_1+i}}, \ \ \forall i \in\{1,\ldots,N_2\}, 
			$$
			then  
			$$
			M_{N_1+N_2}(\delta_1,\ldots,\delta_{N_1},\delta_{N_1+1},\ldots,\delta_{N_1+N_2}) = 	M_{N_1}(\delta_1,\ldots,\delta_{N_1}); 
			$$
			\item[-] \textsc{Responsiveness to contradiction:} for every integers $N_1$ and $N_2$, if $\delta_1,\ldots, \delta_{N_1},\allowbreak \delta_{N_1+1},\ldots,\allowbreak \delta_{N_1+N_2}$ are such that 
			$$
			\exists \, i\in\{1,\ldots,N_2\}: \;\; M_{N_1}(\delta_1,\ldots,\delta_{N_1}) \notin \calZ_{\delta_{N_1+i}},
			$$ 
			then
			$$
			M_{N_1+N_2}(\delta_1,\ldots,\delta_{N_1},\delta_{N_1+1},\ldots,\delta_{N_1+N_2}) 	\neq M_{N_1}(\delta_1,\ldots,\delta_{N_1}).
			$$	
			\hfill$\star$
		\end{itemize}
	\end{lemma}
	\begin{proof}
		In this proof, we use the notation $(\theta^{\ast}_{\widehat{A},N_1},\calL^{\ast}_{N_1})$ to indicate $z^\ast_{N_1} = M_{N_1}(\delta_1,\ldots,\delta_{N_1})$ and $(\theta^{\ast}_{\widehat{A},N_1+N_2},\calL^{\ast}_{N_1+N_2})$ to indicate $z^\ast_{N_1+N_2} = M_{N_1+N_2}(\delta_1,\ldots,\delta_{N_1},\delta_{N_1+1},\ldots,\delta_{N_1+N_2})$. The three properties are proved in turn. \\
		\\
		\noindent \textsc{Permutation invariance:} this is obvious, because the order in which data points appear in the sample $\cal D$ does not affect $z^\ast_N$.\\
		\\
		\textsc{Stability in the case of confirmation:} consider the optimization program
		\begin{align} \label{pb:aux}
			\min_{\theta \; \atop \xi_i \geq 0, i=1,\ldots,N_1+N_2} & \quad  c(\theta) + \rho \sum_{i=1}^{N_1+N_2} \xi_i \\
			\textrm{\rm subject to:} & \quad f(\theta,\tilde{\delta}_i^{(j)}) \leq \xi_i, \ \  j=1,\ldots, M;\,\,i = 1, \ldots ,N_1, \nonumber
		\end{align}
		Problem \eqref{pb:aux} is the same as problem \eqref{svr} except that $N$ has been replaced by $N_1$ and that there are extra variables $\xi_{N_1+1},\ldots,\xi_{N_1+N_2}$, which however are ineffective since they only appear in the cost and are set to zero at optimum. Thus, the solution to \eqref{pb:aux} is $(\theta^{\ast}_{\widehat{A},N_1},\xi^{\ast}_{\widehat{A},N_1,1},\ldots,\xi^{\ast}_{\widehat{A},N_1,N_1},0,\ldots,0)$, i.e., it is the solution to \eqref{svr} with $N_1$ in place of $N$ complemented with extra variables $\xi_i$ that are zero for any $i = N_1+1, \ldots, N_1+N_2$. Now, if the premise formulated in ``Stability in the case of confirmation'' is true, then $(\theta^{\ast}_{\widehat{A},N_1},\xi^{\ast}_{\widehat{A},N_1,1},\ldots,\xi^{\ast}_{\widehat{A},N_1,N_1},0,\ldots,0)$ is also the solution to \eqref{svr} with $N_1+N_2$ in place of $N$ because \eqref{svr} with $N_1+N_2$ in place of $N$ is the same as program \eqref{pb:aux} with the addition of constraints that are already satisfied by the solution to \eqref{pb:aux} (indeed, condition $M_{N_1}(\delta_1,\ldots,\delta_{N_1}) \in \calZ_{\delta_{N_1+i}}$ yields $f(\theta^{\ast}_{\widehat{A},N_1},\tilde{\delta}^{(j)}_{N_1+i}) \leq 0$ for all $j=1,\ldots,M$). This implies that $\theta^{\ast}_{\widehat{A},N_1+N_2}$ and $\theta^{\ast}_{\widehat{A},N_1}$ coincide. Once this is recognized, then $\calL^{\ast}_{N_1+N_2} = \calL^{\ast}_{N_1}$ easily follows because none of the $\delta_{N_1+i}$ are violated by $\theta^{\ast}_{\widehat{A},N_1+N_2} = \theta^{\ast}_{\widehat{A},N_1}$ and, therefore, none of them have to be placed in  $\calL^{\ast}_{N_1+N_2}$. This shows that $z^\ast_{N_1+N_2} = z^\ast_{N_1}$ and closes this point. \\
		\\
		\textsc{Responsiveness to contradiction:} after adding $\delta_{N_1+1},\ldots,\delta_{N_1+N_2}$ to $\delta_1,\ldots,\delta_{N_1}$, two cases may arise: either $\theta^{\ast}_{\widehat{A},N_1+N_2} \neq \theta^{\ast}_{\widehat{A},N_1}$ or $\theta^{\ast}_{\widehat{A},N_1+N_2} = \theta^{\ast}_{\widehat{A},N_1}$. In the first case, $z^\ast_{N_1+N_2}$ is different from $z^\ast_{N_1}$ because the $\theta$ components are not the same. If instead $\theta^{\ast}_{\widehat{A},N_1+N_2} = \theta^{\ast}_{\widehat{A},N_1}$, then the $\delta_{N_1+i}$'s that are violated by $\theta^{\ast}_{\widehat{A},N_1}$ must enter the multiset $\calL^{\ast}_{N_1+N_2}$ because they are also violated by $\theta^{\ast}_{\widehat{A},N_1+N_2} = \theta^{\ast}_{\widehat{A},N_1}$. Under the premise formulated in ``Responsiveness to contradiction'', this happens for at least one of the $\delta_{N_1+i}$, and this implies that $\calL^{\ast}_{N_1+N_2} \neq \calL^{\ast,}_{N_1}$. Thus, in any case we have that $z^\ast_{N_1+N_2} \neq z^\ast_{N_1}$ and this concludes this last point. 
	\end{proof}
	
	Before moving to Lemma \ref{main-prop:nondegeneracy}, we state a simple proposition, which is instrumental to the proof of the lemma. 
	
	\begin{proposition} \label{prop:non-acc}
		Assumption \ref{non-acc} implies that: for every $\theta$, it holds that 
		\begin{equation}
			\label{relation non accumulation}
			\prob \left\{\delta : \; \exists \, \tilde{\delta}^{(j)} \in \widehat{A}_\delta \text{ such that } f(\theta,\tilde{\delta}^{(j)}) = 0 \right\} = 0.
		\end{equation} 	\hfill $\star$
	\end{proposition}
	\begin{proof}
		The following chain of equalities holds true 
		\begin{eqnarray*}
			\lefteqn{\prob \left\{ \delta : \; \exists \, \tilde{\delta}^{(j)} \in \widehat{A}_\delta \text{ such that } f(\theta,\tilde{\delta}^{(j)}) = 0 \right\} } \\
			&\leq & \sum_{j=1}^M \prob \left\{ f(\theta,\tilde{\delta}^{(j)}) = 0 \right\} \\ 
			& = & \sum_{j=1}^M \prob \left\{ |\tilde{y}^{(j)} - w^\top \tilde{u}^{(j)}  -b|-\gamma = 0 \right\}  \\ 
			& = & \sum_{j=1}^M \mathbb{E} \left[ \prob \left\{ |\tilde{y}^{(j)} - w^\top \tilde{u}^{(j)}  -b|-\gamma = 0  \; \Big| \; u \right\} \right] \\
			& = & \sum_{j=1}^M \mathbb{E} \left[ \prob \left\{ y =    - \tilde{d}_y^{(j)} + w^\top u + w^\top \tilde{d}_u^{(j)}  + b \pm \gamma \; \Big| \; u \right\} \right]. 
		\end{eqnarray*}
		The last term is equal to zero because, for each $j$ and any fixed $u$, quantities $- \tilde{d}_y^{(j)} + w^\top u + w^\top \tilde{d}_u^{(j)}  + b - \gamma$ and $- \tilde{d}_y^{(j)} + w^\top u + w^\top \tilde{d}_u^{(j)}  + b + \gamma$ are constant, and, thanks to Assumption~\ref{non-acc}, the conditional probability that $y$ takes any predetermined value given $u$ is zero.
	\end{proof}
	
	\begin{lemma}[Non-degeneracy of $M_N$ and complexity evaluation] \label{main-prop:nondegeneracy}
		The family of maps $M_N$, $N=0,1,\ldots$, satisfies Assumption 4 in \cite{Garatti22}, namely: for every $N$, with probability $1$ the decision $M_N(\delta_1,\ldots,\delta_N)$ coincides with the decision $M_k(\delta_{i_1},\ldots,\delta_{i_k})$, where $\delta_{i_1},\ldots,\delta_{i_k}$ are the support elements of $M_N(\delta_1,\ldots,\delta_N)$. Moreover, with probability $1$ the number of support elements of $M_N(\delta_1,\ldots,\delta_N)$ is equal to the {\em adversarial complexity} $s^\ast_{A,\widehat{A}}$ (Definition \ref{def:adv_complex_outer}). \hfill$\star$
	\end{lemma}
	\begin{proof}
		The proof is obvious when $N=0$ since, when there are no data points, there are no support elements either, and the statement of the lemma boils down to the tautology $M_0 = M_0$. \\
		\\
		Consider thus the case $N \geq 1$. We want to precisely characterize the support elements of $z^\ast_N$. \\
		\\
		Firstly, notice that, for all $(\delta_1,\ldots,\delta_N) \in \Delta^N$, it is always the case that all the $\delta_i$'s such that $f(\theta^\ast_{\widehat{A}},\tilde{\delta}_i) \leq 0$ for all $\tilde{\delta}_i \in A_{\delta_i}$ and $f(\theta^\ast_{\widehat{A}},\tilde{\delta}_i^{(j)}) < 0$ for all $\tilde{\delta}_i^{(j)} \in \widehat{A}_{\delta_i}$ are not support elements of $z^\ast_N$. The reason for this is that each of these $\delta_i$'s corresponds to $M$ constraints in \eqref{svr} that are non-active at optimum. 
		Thus, owing to convexity, if $\delta_i$ is removed, then $\theta^\ast_{\widehat{A}}$ does not change; consequently, $\calL^\ast$ does not change either because $\delta_i$ was not included in the $\calL^\ast$ constructed from $\delta_1,\ldots,\delta_N$. This shows that  $M_{N-1}(\delta_1,\ldots,\delta_{i-1},\delta_{i+1},\ldots,\delta_N) = M_{N}(\delta_1,\ldots,\delta_N)$.\\
		\\
		Secondly, all the $\delta_i$'s such that $f(\theta^\ast_{\widehat{A}},\tilde{\delta}_i) > 0$ for at least  one $\tilde{\delta}_i \in A_{\delta_i}$ or $f(\theta^\ast_{\widehat{A}},\tilde{\delta}_i^{(j)}) > 0$ for at least one $\tilde{\delta}_i^{(j)} \in \widehat{A}_{\delta_i}$ are always support elements of $z^\ast_N$. As a matter of fact, when one of these $\delta_i$'s is removed from $\delta_1,\ldots,\delta_N$, then either $\theta^\ast_{\widehat{A}}$ changes or, if $\theta^\ast_{\widehat{A}}$ does not change, then $\calL^\ast$ has to change because there is one less $\delta_i$ that was previously included in the $\calL^\ast$ constructed from $\delta_1,\ldots,\delta_N$. In both cases, $M_{N-1}(\delta_1,\ldots,\delta_{i-1},\delta_{i+1},\ldots,\delta_N) \neq M_{N}(\delta_1,\ldots,\delta_N)$. \\
		\\
		The only case left is when a $\delta_i$ is such that:  $f(\theta^\ast_{\widehat{A}},\tilde{\delta}_i) \leq 0$ for all $\tilde{\delta}_i \in A_{\delta_i}$ and $f(\theta^\ast_{\widehat{A}},\tilde{\delta}_i^{(j)}) \leq 0$ for all $\tilde{\delta}_i^{(j)} \in \widehat{A}_{\delta_i}$, but
		\begin{equation} \label{active-condition}
			f(\theta^\ast_{\widehat{A}},\tilde{\delta}_i^{(j)}) = 0 \text{ for at least one } \tilde{\delta}_i^{(j)} \in \widehat{A}_{\delta_i}.
		\end{equation}
		It is claimed that these $\delta_i$'s are support elements with probability $1$. This fact is proven by contradiction: suppose that one such $\delta_i$ is not a support element, i.e., $M_{N-1}(\delta_1,\ldots,\delta_{i-1},\allowbreak \delta_{i+1},  \ldots,\delta_N) = M_{N}(\delta_1,\ldots,\delta_N)$. This implies that $\theta^{\ast,(i)}_{\widehat{A}} = \theta^\ast_{\widehat{A}}$, where $\theta^{\ast,(i)}_{\widehat{A}}$ is the solution to the optimization program obtained from \eqref{svr} when the $\xi_i$ variable and the constraints corresponding to $\delta_i$ are discarded. In view of \eqref{active-condition}, we then have $f(\theta^{\ast,(i)}_{\widehat{A}},\tilde{\delta}_i^{(j)}) = 0 \text{ for at least one } \tilde{\delta}_i^{(j)} \in \widehat{A}_{\delta_i}$, which, however, only occurs with probability zero, owing to Proposition \ref{prop:non-acc} and the independence of $\delta_i$ from $\delta_1,\ldots,\delta_{i-1},\delta_{i+1},\ldots,\delta_N$.\footnote{The reason why independence is advocated is that $\theta^{\ast,(i)}_{\widehat{A}}$ is constructed from $\delta_1,\ldots,\delta_{i-1},\delta_{i+1},\ldots,\delta_N$ and, owing to independence, $\theta^{\ast,(i)}_{\widehat{A}}$ can be treated as deterministic (as is in Proposition \ref{prop:non-acc}) when considering the variability of $\delta_i$.} This shows that the $\delta_i$'s considered in this last,  third, case are all of support with probability $1$. \\
		\\
		Wrapping up, we have seen that, with probability $1$, the support elements $\delta_{i_1},\ldots,\delta_{i_k}$ of $\delta_1,\ldots,\delta_N$ are the $\delta_i$'s that satisfy conditions (i)-(iii) in Definition \ref{def:adv_complex_outer}; the number of these elements is $s^\ast_{A,\widehat{A}}$. Moreover, it holds that $M_N(\delta_1,\ldots,\delta_N) = M_k(\delta_{i_1},\ldots,\delta_{i_k})$ because removing from \eqref{svr} all constraints but those given by $\delta_{i_1},\ldots,\delta_{i_k}$ corresponds to dropping constraints that are non-active at the optimum: this leaves $\theta^\ast_{\widehat{A}}$ unaltered and  also $\calL^\ast$ does not change because in $\delta_{i_1},\ldots,\delta_{i_k}$ there are all the $\delta_i$'s that contribute to forming $\calL^\ast$ when all the $\delta_1,\ldots,\delta_N$ are in place. \\
		\\
		This concludes the proof of Lemma \ref{main-prop:nondegeneracy}. 
	\end{proof}
	
	We are now in the position to conclude the proof of Theorems \ref{main_th_standard} and \ref{main_th_outer}. Lemmas~\ref{main-prop:consistency} and~\ref{main-prop:nondegeneracy} show  that the assumptions of Theorem 2 of \cite{Garatti22} are verified, and an application of this theorem, along with the fact that the number of support elements of $z^\ast_N$ is equal with probability~$1$ to $s^\ast_{A,\widehat{A}}$, yields
	\begin{equation} \label{guaranteedviolation}
		\prob^N \left\{ \underline{\eps}(s^\ast_{A,\widehat{A}}) \leq V(z^\ast_N)\leq \overline{\eps}(s^\ast_{A,\widehat{A}}) \right\} \geq  1-\beta. 
	\end{equation}
	As we have already noticed, when $\widehat{A} \subseteq A$ it holds that $\textnormal{Risk}_A(\theta) = V(z)$ for every $z$, from which we have that $\textnormal{Risk}_A(\theta^\ast_{\widehat{A}}) = V(z^\ast_N)$ for every $\delta_1,\ldots,\delta_N$. This means that \eqref{guaranteedviolation} can be rewritten as
	$$
	\prob^N \left\{ \underline{\eps}(s^\ast_{A,\widehat{A}}) \leq \textnormal{Risk}_A(\theta^\ast_{\widehat{A}}) \leq \overline{\eps}(s^\ast_{A,\widehat{A}}) \right\} \geq  1-\beta,
	$$
	which proves Theorem \ref{main_th_standard}. \\
	
	When instead $\widehat{A} \not \subseteq A$, the weaker relation holds that $\textnormal{Risk}_A(\theta^\ast_{\widehat{A}}) \leq V(z^\ast_N)$ for every $\delta_1,\ldots,\delta_N$. 
	Hence, from \eqref{guaranteedviolation} we obtain 
	$$
	\prob^N \left\{ \textnormal{Risk}_A(\theta^\ast_{\widehat{A}}) \leq \overline{\eps}(s^\ast_{A,\widehat{A}}) \right\} \geq \prob^N \left\{ V(z^\ast_N) \leq \overline{\eps}(s^\ast_{A,\widehat{A}}) \right\} \geq  1-\beta,
	$$
	thus proving Theorem \ref{main_th_outer}. \qed
	
	\section{Proof of Theorems \ref{th:general1} and \ref{th:general2}} 
	\label{sec:derivations-extension} 
	
	At the beginning of the proof of Theorems \ref{main_th_standard} and \ref{main_th_outer}, we have been well-advised to introduce a notation that has general validity and applies to the case of Theorems~\ref{th:general1} and~\ref{th:general2} as well. As a consequence, the proof of Theorems \ref{main_th_standard} and \ref{main_th_outer} immediately extends to cover Theorems \ref{th:general1} and \ref{th:general2} under the notice that: (i) instead of $\calP(\theta)$ one writes $\{ \delta \in \Delta: \; f(\theta,\delta) \leq 0\}$; (ii) $\widehat{A} \subseteq A$ is replaced by $\widehat{A}_\delta \subseteq A_\delta$ for all $\delta \in \Delta$; (iii) Proposition \ref{prop:non-acc} is skipped and, whenever Proposition \ref{prop:non-acc} is invoked, Assumption \ref{non-acc-general} is used instead. \qed
	
	
	\section{Acknowledgments}
		
	The research presented in this article has been partly supported by FAIR (Future Artificial Intelligence Research) project, funded by the NextGenerationEU program within the PNRR-PE-AI scheme (M4C2, Investment 1.3, Line on Artificial Intelligence), by the PRIN 2022 project 2022RRNAEX ``The Scenario Approach for Control and Non-Convex Design'' (CUP: D53D23001440006), funded by the NextGeneration EU program (Mission 4, Component 2, Investment 1.1), and by the PRIN PNRR project  P2022NB77E “A data-driven cooperative framework for the management of distributed energy and water resources” (CUP: D53D23016100001), funded by the NextGeneration EU program (Mission 4, Component 2, Investment 1.1).

	\appendix
	
	\newpage
	
	\section{Framing the construction of the convex hull of points in ${\mathbb R}^{2}$ within the setup of \eqref{opt-relax}}\label{appendix convex hull} 
	Paper \cite{Ratdstrom1952} proves that the set of closed convex sets in ${\mathbb R}^{2}$ (with Minkowski sum, $K_1 + K_2 = \{k_1 + k_2 \mbox{ with } k_1 \in K_1, k_2 \in K_2\}$ and product by  a scalar defined as $\lambda K = \{\lambda k, \lambda \in {\mathbb R}, k \in K\}$) can be embedded as a convex cone in a real linear vector space. In the formulation \eqref{opt-relax}, this cone corresponds to the domain $\Theta$. In what follows, we show that function  $f(\theta,\delta) = \min_{x \in \theta} \mathrm{dist}(x,\delta)$, which coincides with the function that ``is zero when the point $\delta$ is in the convex set $\theta$, and takes a value that grows linearly with the distance between the point and the convex set when the point is outside'', is convex for any $\delta$, and so is the perimeter of the convex set $\theta$, which we take as cost functional $c(\theta)$ (this fully aligns the construction in Example \ref{convex hull-example} with the setup of \eqref{opt-relax}; the fact that these choices lead to constructing the convex hull is instead left as an exercise to the reader). Convexity of $c(\theta)$ follows from the fact that the perimeter is in fact linear in $\theta$, see, e.g., point 4-8 in \cite{YaglomBoltyanskii1961}. Instead, the convexity of $f(\theta,\delta)$ follows from this chain of equations: $f(\alpha\theta_1 + (1-\alpha)\theta_2,\delta) = \min_{x \in \alpha\theta_1 + (1-\alpha)\theta_2} \mbox{dist}(x,\delta) = \min_{x_1 \in \theta_1, x_2 \in \theta_2} \mbox{dist}(\alpha x_1 + (1-\alpha)x_2,\delta) \leq \min_{x_1 \in \theta_1, x_2 \in \theta_2} \left[ \alpha \mbox{dist}(x_1,\delta) + (1-\alpha)\mbox{dist}(x_2,\delta) \right] = \alpha \min_{x_1 \in \theta_1} \mbox{dist}(x_1,\delta) + (1-\alpha)\min_{x_2 \in \theta_2} \mbox{dist}(x_2,\delta) = \alpha f(\theta_1,\delta) + (1-\alpha)f(\theta_2,\delta)$.


\bibliographystyle{plain}
\bibliography{adversarial_arxiv}

\end{document}